\documentclass[preprint,12pt]{elsarticle}
\usepackage[utf8]{inputenc}
\usepackage{algorithm}
\usepackage{algpseudocode}
\usepackage{amsmath}
\usepackage{amsfonts}
\usepackage{amssymb}
\usepackage{amsthm}
\usepackage{amsmath}
\usepackage{caption}
\usepackage[dvipsnames]{xcolor}
\usepackage{graphicx}
\usepackage[margin=2.5cm]{geometry} 
\usepackage{hyperref} 
\usepackage{cleveref}
\usepackage{subcaption}
\usepackage{tikz} 
\usepackage{tikz-cd}
\usepackage{tikz-3dplot}
\usepackage{subcaption}
\usepackage{makecell}

\DeclareMathOperator*{\argmin}{arg\,min}

\newtheorem{remark}{Remark}[section]
\newtheorem{proposition}{Proposition}[section]
\newtheorem{definition}{Definition}[section]
\newtheorem{theorem}{Theorem}[section]
\newtheorem{hypothesis}{Hypothesis}[section]
\newtheorem{example}{Example}[section]

\begin{document}

\begin{frontmatter}


\title{Estimating Dataset Dimension via Singular Metrics under the Manifold
Hypothesis:
Application to Inverse Problems}

\author{Paola Causin\corref{cor1}}
\ead{paola.causin@unimi.it}
\author{Alessio Marta}
\affiliation{organization={Dept. of Mathematics, University of Milan},
           addressline={via Saldini 50}, 
            city={Milano},
            postcode={20133}, 
            country={Italy}}

\begin{abstract}
High-dimensional datasets often exhibit low-dimensional geometric structures, as suggested by the manifold hypothesis, which implies that data lie on a smooth manifold embedded in a higher-dimensional ambient space. While this insight underpins many advances in machine learning and inverse problems, fully leveraging it requires to deal with three key tasks: estimating the intrinsic dimension (ID) of the manifold, constructing appropriate local coordinates, and learning mappings between ambient and manifold spaces. In this work, we propose a  framework that addresses all these challenges using a Mixture of Variational Autoencoders (VAEs) and tools from Riemannian geometry. We specifically focus on estimating the ID of datasets by analyzing the numerical rank of the VAE decoder pullback metric. The estimated ID guides the construction of an atlas of local charts using a mixture of invertible VAEs, enabling accurate manifold parameterization and efficient inference. We  how this approach enhances solutions to ill-posed inverse problems, particularly in biomedical imaging, by enforcing that reconstructions lie on the learned manifold. Lastly, we explore the impact of network pruning on manifold geometry and reconstruction quality, showing that the intrinsic dimension serves as an effective proxy for monitoring model capacity.
\\[2mm]
\end{abstract}

\end{frontmatter}

\section{Introduction}

High-dimensional datasets are ubiquitous in modern scientific and engineering domains, yet their inherent complexity often poses significant computational challenges~\cite{Chandra2001}. A prevailing assumption, known as the {\em manifold hypothesis}, assumes that such data frequently resides on a lower-dimensional, smooth manifold $M$ of dimension $d$ embedded within a higher-dimensional Euclidean space $\mathbb{R}^D$ (ambient space), with $d \ll D$. That is to say, only~$d$ independent
coordinates are needed, at least locally, to accurately
describe the manifold. To exploit this structure, having access to 
data only in the ambient space, one should: 
(a) estimate the dimension~$d$; (b) construct accordingly a coordinate
system describing points 
in~$M$; 
(c) obtain the mapping functions from the ambient
space to the manifold and vice-versa. 
However, this procedure is not trivial and often just one of the above points is addressed~\cite{Zeng_2024}.

\null

\noindent The task of estimating the number of degrees of freedom required to represent a dataset has been of interest to different fields ranging from 
information science to statistics and compressed sensing for many years. For example, classical approaches as PCA and its variants perform  dimensionality reduction by projecting
data on a linear subspace. 
At the same time, they also provide the mapping
functions to the low--dimensional representation. 
However, the linearity assumption is often too restrictive~\cite{Lever2017}. The manifold hypothesis suggests indeed that data lies on a nonlinear, low-dimensional manifold embedded in high-dimensional space. 
Newer approaches based on neural architectures are expected to be able to learn nonlinear manifolds, along with their geometrical structure, directly via deep generative models with improved accuracy.
In this work, we explore this idea and we propose a workflow which addresses points (a),(b) and (c) above. We specifically focus on the study of the structure of the learned manifold and we introduce a novel methodology for estimating its intrinsic dimension (ID). Our approach leverages singular Riemannian metrics induced by Variational Autoencoder (VAE) neural architectures. The core idea involves utilizing the pullback metric through the VAE decoder to recover the local geometry of the underlying manifold. The ID is then robustly estimated by analyzing the numerical rank of the resulting (potentially degenerate) metric tensor. This technique is particularly pertinent to the solution of inverse problems in biomedical imaging, where observations are typically corrupted by noise and the reconstruction process is severely ill-posed. By constraining solutions to lie on a learned manifold, we can effectively address both the ill-posedness and high dimensionality of these problems. The estimated intrinsic dimension further guides the construction of an atlas of local charts using mixtures of invertible VAEs, ensuring accurate manifold parameterization and tractable inference. Additionally, we investigate the effects of network pruning on the learned manifold structure and reconstruction quality, demonstrating that significant model compression is feasible up to a critical threshold related to the ID itself, beyond which the manifold geometry and reconstruction fidelity degrade.

\subsection{Related work}

\subsubsection{Estimation of the intrinsic dimension of datasets}
Several algorithms exist for estimating the ID of a generic dataset~$X$
embedded in a Euclidean space~$\mathbb{R}^D$. These methods can be broadly categorized into four groups: projective methods, fractal-based approaches, nearest-neighbor techniques, and sim\-plex-based strategies.
Projective methods rely on computing the eigenvalues of the $D \times D$ covariance matrix of the dataset, which are then used to estimate the intrinsic dimension by analyzing the magnitude of the sorted eigenvalues via Principal Component Analysis (PCA). Both a global version (gPCA) and a multiscale version (mPCA) have been proposed~\cite{5278634,2017_Little}. The main limitation of the former is that it can correctly detect the dimension only for linearly embedded manifolds and tends to overestimate the dimension of curved or non-linearly embedded datasets. In contrast, the latter addresses this issue but it may require in certain applications a very large number of samples~\cite{ErbaGherardi}, which can be challenging to achieve. 
Fractal methods are based on the assumption that the points of the dataset are sampled from the underlying manifold according to a smooth probability density function. The prototypical algorithm in this class is the correlation dimension method, introduced by Grassberger and Procaccia in \cite{GRASSBERGER1983189} to measure the fractal dimension in dynamical systems. In their approach, the dimension of a dataset is inferred using the so-called density of neighbors:
$\rho(r) = {2}/(N(N-1)) \sum_{1 \leq a<b\leq N}  H(r - \| x^a - x^b \|)$,
where $H$ is the step function, $x^a$ is the $a$-th point of the dataset, $N$ the number of samples and $\rho(r) \sim r^d$ as $r \rightarrow 0$. Since the original work of Grassberger and Procaccia, many other fractal-based methods for estimating intrinsic dimension have been developed \cite{1039212,3470,10.1155-2015-759567,213fadd85c594e5d96efda8eb455ab0c}. While effective in low-dimensional settings ($d \precsim 10$), these methods are known to underestimate the dimension of higher-dimensional manifolds \cite{CERUTI20142569}. Indeed, \cite{ECKMANN1992185} showed that accurately estimating the intrinsic dimension via fractal methods requires a number of samples that grows exponentially with the intrinsic dimension $d$.
Nearest neighbor-based approaches, such as MLE~\cite{Levina04,Karbauskaite_Dzemyda_Mazetis_2011}, estimate the intrinsic dimension under the assumption that nearby points are uniformly sampled from $d$-dimensional balls of sufficiently small radii, where $d$ is the true data dimension. However, when the true intrinsic dimension is high, these methods are known to significantly underestimate it \cite{CERUTI20142569}. To overcome this issue, several variants have been proposed, including MIND${}_{KL}$ \cite{Lombardi11}, DANCo \cite{CERUTI20142569}, and GeoMLE \cite{pmlr-v101-gomtsyan19a}.
Simplex-based methods, such as Expected Simplex Skewness \cite{6866171}, analyze the geometric properties of simplices constructed from the dataset. These methods generally yield good performance and require only mild assumptions about the underlying manifold. However, they tend to be more sensitive to noise than projective methods \cite{8404f72ee760436dad7f1be15af4b3d1}.
In addition to the aforementioned groups, more recent intrinsic dimension estimation techniques have been proposed, including methods based on stochastic processes \cite{Tempczyk22,Kamkari24,Candelori25} and others employing normalizing flows \cite{NEURIPS2022_4f918fa3}.

\subsubsection{Manifold learning via mixture models}
Nonlinear dimensionality reduction has traditionally been performed using algorithms such as Sammon's mapping, nonlinear PCA, Laplacian eigenmaps, Isomap, curvilinear component analysis, diffusion maps, and local tangent space alignment \cite{DBLP:books/sp/GhojoghCKG23,Ma2011,annurev-statistics-040522-115238}. More recently, manifold learning has emerged as a key area of interest within deep learning. Many successful approaches have been based on generative models, including generative adversarial networks \cite{9878714,DBLP:conf/nips/KhayatkhoeiSE18,doi:10.1073/pnas.2311436121}, injective flows \cite{NEURIPS2023_572a6f16,caterini2021rectangularflowsmanifoldlearning}, and score-based diffusion models \cite{10.5555/3600270.3602868,NIPS-eSZqaIrDLZR}.
In this work, we focus on VAEs~\cite{10.5555/3172077.3172161,pineau2018,pmlr-v130-connor21a} but different
architectures could be employed as well. A single VAE can be used to learn a manifold only when the manifold admits a global parameterization, which in particular requires the data to be connected and without topological holes. To address more complex structures, some authors have proposed modeling the VAE latent space using a Gaussian mixture model, enabling the learning of multimodal probability distributions \cite{falck2021multifacet,10.5555/3172077.3172161,10.5555/3600270.3601411,pineau2018}. However, this approach struggles to model overlapping charts and is thus not adequate for parameterizing manifolds with holes, such as a torus, or rotational degrees of freedom. 
To overcome this limitation, Alberti et al.~\cite{ManifoldVAEs} proposed approximating the data manifold using a mixture model of VAEs, a framework in which each component of the mixture is implemented by a separate VAE \cite{locatello2019,9408405}. Following this approach, we construct a collection of charts (an atlas) using a mixture model of VAEs, where each component of the mixture represents a local chart of the manifold.

\subsubsection{Manifold learning and inverse problems}

Many problems in image processing and various physical applications can be formulated as inverse problems of the form $y = F(x) + \zeta$ where the forward operator $F$ is a possibly nonlinear map between two Euclidean spaces, $y$ are the observations, $x$ the quantity to be reconstructed and $\zeta$ a noise term. Reconstructing $x$ solely from observations $y$ is generally not feasible due to the high dimensionality of the problem and its inherent ill-posedness. A common strategy to overcome these challenges is to incorporate prior knowledge through regularization theory. This involves minimizing the sum of a data fidelity term $\mathcal{L}$—such as the $L^2$ reconstruction error—and a data-driven regularizer $\mathcal{R}$ that encodes the prior information \cite{Hertrich_2022,Altekrger2022PatchNRLF,arridge_maass,Asim2018BlindID,pmlr-v70-bora17a}. In particular, high dimensionality of the problem can be mitigated by identifying a low-dimensional representation of the data to be reconstructed. Assuming that the inputs $x$ lie on a low-dimensional manifold, we can then formulate the constrained optimization problem
k$$\argmin_{x} \mathcal{L}(F(x),y) \ \textit{subject to} \ x \in M,$$
which corresponds to a regularizer $\mathcal{R}$ that is zero for $x$ on the manifold $M$ and infinite otherwise. This approach, however, requires learning the manifold $M$ beforehand. 
In \cite{ManifoldVAEs}, this is achieved using a mixture of invertible VAEs, which requires specifying the dimension of the latent space (and thus of the manifold) a priori. To address this limitation, we propose a method for estimating the ID of the data manifold via singular Riemannian geometry, using first a VAE trained to perform a Whitney embedding of the data which performs an initial dimensionality reduction and a series of AEs that build the local charts, significantly accelerating the process—especially when a large number of charts is needed.

\subsection{Original contributions and organization of the paper}

The main contribution of this paper is a novel methodology to perform manifold learning while 
at the same time estimating the manifold intrinsic dimension. The approach uses the singular Riemannian metric induced by a neural architecture as a VAE:  
namely, we use the VAE decoder’s pullback metric to recover the local geometry of the data manifold and we estimate its ID via the numerical rank of (possibly degenerate) metric tensors. 
We construct an atlas of charts on the manifold via mixtures of invertible VAEs informed by the estimated ID, enabling accurate manifold parameterization and tractable inference.
We discuss the benefits of the application of the proposed
approach to the solution of inverse problems from biomedical imaging, where we consider a  regularization strategy based on constraining the solution to belong to the learned manifold.
In addition, we explore the effect of pruning on the net, showing that significant parameter reduction is possible up to a critical threshold, beyond which both reconstruction fidelity and estimated ID degrade—offering ID as a proxy for network capacity monitoring.

\null

The paper is organized as follows. In \Cref{section:preliminaries}, we recall key concepts from Differential and Riemannian Geometry that are essential for the remainder of this work, and we review foundational ideas related to mixtures of VAEs and invertible neural networks. In \Cref{section:dimension_estimate_general}, we present our ID estimation method based on singular Riemannian Geometry, demonstrating how both the encoder and decoder of a VAE can be leveraged for this purpose. We also illustrate the effectiveness of our technique on simple low-dimensional datasets. In \Cref{section:application_to_inverse_problem}, we explore how our method can be applied to an inverse problem from medical imaging, where estimating the ID is crucial for constructing a representation of the data manifold using local charts. These tasks are carried out on a toy medical imaging dataset in \Cref{sec:Numerical_experiments}, where we first perform ID estimation, and then build an atlas of local charts using a mixture of invertible VAEs. We also compare our ID estimation approach with existing methods, demonstrating superior performance on high-dimensional data. Finally, we investigate how pruning a VAE influences the ID of the learned manifold.

\section{Preliminaries}\label{section:preliminaries}
\subsection{Background on Differential Geometry}\label{subsec:geometry}
In this section we briefly recall some differential geometry notions  useful for the following (see, e.g., \cite{doCarmo,TuIntro}). We start from the definition of manifold.
\begin{definition}\label{def:manifold_charts}
A smooth $d$-dimensional manifold $M$ is a second countable and Hausdorff topological space such that every point $p \in M$ has a neighbourhood $U_p$ that is homeomorphic to a subset of $\mathbb{R}^d$ through a map $\phi_p:U_p \rightarrow \mathbb{R}^{d}$, with the additional requirement that for $p, q\in M$ if $U_p \cap U_q \neq \emptyset$, then the 
transition chart $\phi_p \circ \phi_q^{-1}$ is a smooth diffeomorphism. The pair $(U_p,\phi_p)$ is called a local chart and the collection of all the possible local charts at all points is called atlas (see \Cref{fig:manifold}).
\end{definition}
\begin{figure}[tbh]
\centering
\includegraphics[scale=1.28,angle=0]{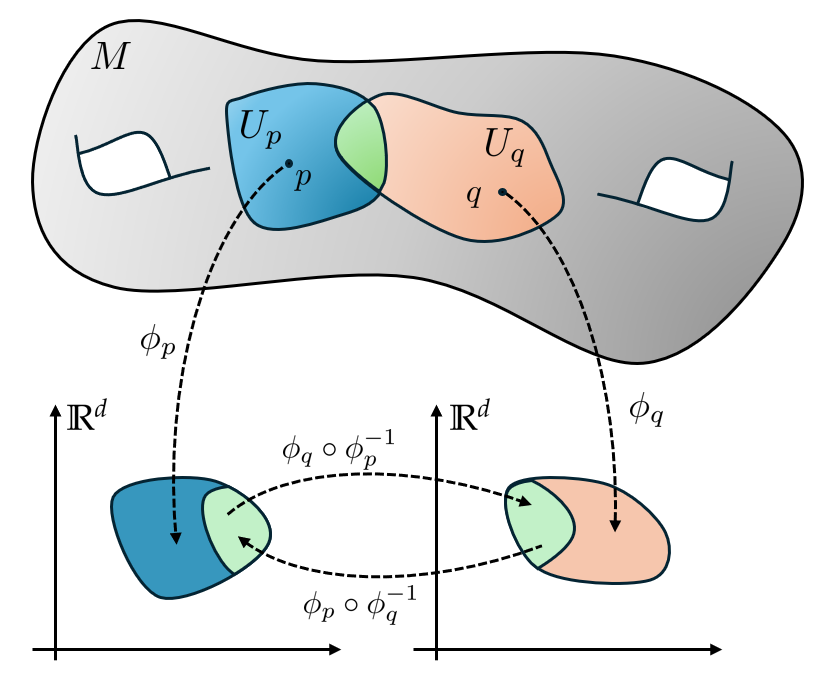}
\caption{Visual representation of a smooth $d$-dimensional manifold $M$. Each point $p \in M$ has a neighbourhood $U_p$ that is homeomorphic to an open subset of $\mathbb{R}^d$ via a map $\phi_p: U_p \rightarrow \mathbb{R}^d$. If two neighbourhoods $U_p$ and $U_q$ overlap (i.e., $U_p \cap U_q \neq \emptyset$), the corresponding transition map $\phi_p \circ \phi_q^{-1}$ is required to be a smooth diffeomorphism. Each pair $(U_p, \phi_p)$ defines a \emph{local chart}, and the collection of all such charts forms an \emph{atlas} of the manifold.}
\label{fig:manifold}
\end{figure}
In practice, a convenient way to build manifolds is to realize them as subsets of Euclidean spaces by means of embeddings maps. 
\begin{definition}
A smooth map $F : M \rightarrow \mathbb{R}^d$ is an embedding if its differential is
everywhere injective and if it is an homeomorphism with its image.
\end{definition}
Any manifold can be realized in a Euclidean space of dimension high enough as made precise by the following~\cite{NashEmbedding}.
\begin{theorem}[Strong Whitney embedding theorem]
Every smooth manifold of dimension~$d$ admits a smooth embedding into $\mathbb{R}^{2d}$.
\end{theorem}
The dimension given by the Whitney embedding theorem is only an upper bound on the minimal dimension of the ambient Euclidean space required to embed a manifold $M$ of dimension $d$. In fact, there are cases where embeddings exist in Euclidean spaces of dimension smaller than $2d$. A simple example is the two-dimensional sphere $\mathbb{S}^2 \subset \mathbb{R}^3$.

Next, we introduce the notion of tangent spaces, which can be viewed as local linear approximations of the manifold in the neighborhood of each point.
\begin{definition}
The tangent space $T_p M$ over a point $p \in M$ is the set of the tangent vectors associated to the curves passing trough $p$. For manifolds realized as subsets of a Euclidean space of dimension $D$, in a local chart $(U_\alpha,\phi_\alpha)$  we have $T_p M = \left\{ J \phi_\alpha^{-1}|_{\phi_\alpha(p)} y \ | \ y \in \mathbb{R}^d \right\}$, where $J$ is the Jacobian matrix of the map.  
\end{definition}
When a manifold is defined as a subset of an ambient Euclidean space, it naturally inherits from it a Riemannian metric, equipping the  tangent space $T_p M$ in each $p \in M$ with the inner product $g_p(u, v) = (u, v)_p$ for all $u, v \in T_p M$. More generally, if the ambient space is equipped with a Riemannian metric, we can pull back it to define a geometrical structure also on $M$.
\begin{definition}[Pullback metric]\label{def:pullback}
``Consider two manifolds $M$ and $N$, the latter equipped with a Riemannian metric $h$. Let $F:M \rightarrow N$\footnote{Here, we use the symbol $N$ to denote the ambient space as usual in this literature. The same symbol has been used above to denote the number of samples, as it is also usual. We will keep this double notation as it will be clear to what it refers in each specific context.} be an embedding. The pullback metric induced by $F$ on the manifold $M$ is given in local coordinates by 
$$g = F^*h =  J_F^T h J_F,$$
where $F^*$ is the pullback operator through
the map $F$ and $J_F:T_pM \rightarrow T_pN$
 is the differential (aka pushforward) of $F$.
\end{definition}
Notice that, since in general a manifold can be defined as an object in its own right, not necessarily as a subset of an ambient space, if we embed a manifold $M$ into another manifold $N$, it may happen that the intrinsic metric $h$ on $M$ and the pullback metric $g$ induced from the ambient space do not coincide. The two metrics agree only for a distinguished class of embeddings.
\begin{definition}
Let $M$ and $N$ be two Riemannian manifolds with metric $g$ and $h$, respectively. An embedding $F:M \rightarrow N$ is called isometric if $F^* h = g$.
\end{definition}
Every Riemannian manifold can be embedded in a Euclidean space of sufficiently high dimension using an \emph{isometric embedding}, that is, a map preserving the Riemannian metric~\cite{NashEmbedding,KUIPER1955545}.
\begin{theorem}[Nash Embedding Theorem]
Let $(M,g)$ be a Riemannian manifold of dimension $d$. Then
there exist a $\mathcal{C}^1$ isometric embedding $F : M \rightarrow \mathbb{R}^{2d+1}$.
\end{theorem}
For isometric embeddings, we can recover the Riemannian metric of $M$ by pulling back the standard Euclidean metric via the embedding map.
Sometimes, a manifold is constructed within an ambient space using a map $F$ that is not an embedding, because its image does not depend on all the variables. However, the restriction of $F$ to a suitable subset of variables may yield an embedding. In general, given a smooth map $F: M \rightarrow N$, where $N$ is a Riemannian manifold equipped with a metric $h$, the pullback metric $F^*h$ is not necessarily a family of non-degenerate inner products on the tangent spaces of $M$. In fact, it may happen that $F^*h$ is degenerate.
\begin{definition}
A degenerate Riemannian metric $g$ over a smooth manifold $M$ is a smooth family $\{g_p\}_{p \in M}$ of degenerate inner products on the tangent spaces of $M$.
\end{definition}
\begin{example}\label{ex:degenerate_metric}
Consider $F: \mathbb{R}^2 \rightarrow \mathbb{R}^3$, $(x,y) \mapsto (-x-y,x+y,(x+y)^3)$. Suppose that $\mathbb{R}^3$ is equipped with the standard Euclidean metric $h = \mathbb{I}$.
A quick computation shows that
\begin{equation*}
J_F (x,y) = 
\begin{pmatrix}
-1 & -1\\
1 & 1\\
3(x+y)^2 & 3(x+y)^2
\end{pmatrix}
\end{equation*} 
from which we find
\begin{equation*}
g(x,y) = (F^* h)(x,y) =\left(  2 + 9(x+y)^4 \right)
\begin{pmatrix}
1 & 1\\
1 & 1
\end{pmatrix}.
\end{equation*}
In this example $F$ is not an embedding since $J_F$ is not injective and the metric $g$ has rank 1, with eigenvalues $\lambda=0$ and $\lambda=4+18(x+y)^4$, corresponding to the eigenvectors $(1,-1)$ and $(1,1)$, respectively. This fact can also be seen considering the change of coordinates $X=x-y$, $Y=x+y$, such that 
\begin{equation*}
g(X,Y) = \left( 2 + 9 Y^4 \right) 
\begin{pmatrix}
0 & 0\\
0 & 1
\end{pmatrix}.
\end{equation*}
Note that the degeneracy of $g$ is actually independent of the choice of the Riemannian 
metric~$h$, as it is instead a consequence of the structure of the map $F$. Indeed, in the new coordinates $(X, Y)$, we observe that $F$ depends only on the variable $X$, which is mapped to $(-X, X, X^3)$. In these coordinates, the restriction of $F$ to the coordinate $X$, namely the function $H(X) = F(X; \cdot) : \mathbb{R} \rightarrow \mathbb{R}^3$, is an isometric embedding.
\end{example}
\begin{remark}
The scenario depicted in this example may occur in a neural architecture as a  VAE and, more in general
in an architecture which produces a latent space of reduced dimension, when certain axes of the latent space are not utilized. This happens when the data reconstruction problem can be effectively solved using fewer dimensions than the prescribed dimension of the latent space.
\end{remark}
Given a map $F: \mathbb{R}^M \rightarrow \mathbb{R}^N$ such that the restriction of $F$ to a suitable submanifold is an embedding, it is generally difficult to determine which (nonlinear) combinations of the variables define that submanifold. However, as illustrated in \Cref{ex:degenerate_metric}, and under the assumption that $F$ has the same rank on the points of the manifold, the number of such variables corresponds to the rank of the Jacobian $J_F$. This rank is equal to the rank of the pullback metric $g = F^* \mathbb{I} = J_F^\top J_F$.

\null

\noindent In light of the above considerations, we make the following assumption on the data.
\begin{hypothesis}
The data, given as points in a high-dimensional Euclidean space $\mathbb{R}^D$, lie on a low-dimensional Riemannian manifold of dimension $d$ embedded in $\mathbb{R}^d$, with $D > 2d + 1$. 
\end{hypothesis}

\subsection{Background on Neural Architectures}\label{subsec:nnarch}
In this section we introduce the neural architectures we will deal with in the following. 

\subsubsection{VAEs, $\beta$--VAEs and Mixture of VAEs}\label{subsec:mixture}

VAEs are neural models that learn to encode input data into a probabilistic latent space and then decode from that space to reconstruct or generate new data. More specifically, given data points $x_1, \dots, x_N$,  sampled from a probability distribution $P_X$ of a random variable~$X$ taking values in an ambient space $\mathbb{R}^D$, a VAE aims to approximate $P_X$ by introducing a lower-dimensional latent random variable $Z$ defined on $\mathbb{R}^m$.
The model consists of an encoder $\mathcal{E} : \mathbb{R}^D \rightarrow \mathbb{R}^m$ and a decoder $\mathcal{D} : \mathbb{R}^m \rightarrow \mathbb{R}^D$, which are trained to be approximate inverses of each other. 
The encoder approximates $Z$ from $X$ by learning a random variable $Z = \mathcal{E}(X) + \zeta_\mathcal{E}$, with $\zeta_\mathcal{E}$ sampled from a Gaussian distribution.
The decoder learns to reconstruct $X$ from $Z$ through a random variable $X_\mathcal{D} = \mathcal{D}(Z)$.

Denoting by $\Theta$ the parameters of the encoder and decoder, the VAE is trained by minimizing the sum of the negative evidence lower bound (ELBO) over all data points
\begin{equation*}
\mathcal{L}(\Theta) = - \sum_{i=1}^N ELBO(x_i|\Theta)
\end{equation*}
where
\begin{equation}
ELBO(x|\Theta) = - \mathbb{E}_{z \sim P_{Z_\mathcal{E}}(\cdot,|x)} \left[ \log P_{X_\mathcal{D}} (x|z)\right] - D_{KL}\left( P_{Z_\mathcal{E}}(\cdot|x) || P_{X_\mathcal{D}}(\cdot)  \right)
\label{eq:ELBO}
\end{equation}
where $P_{X_\mathcal{D}}$ is the computable approximation of
$P_{X}$ and $D_{KL}$ is the Kullback-Leibler (KL) divergence and where one can show that the first term in (\ref{eq:ELBO}) is equal, up to an additive constant, to minus the reconstruction error $\frac12 \| x - \mathcal{D}(z) \|_2^2$. 

\null

The ELBO can be modified adding a real hyperparameter $\beta>0$ weighting the relative importance of the reconstruction term and of the KL divergence \cite{higgins2017betavae}
\begin{equation*}
ELBO(x|\Theta) = - \mathbb{E}_{z \sim P_Z(\cdot,|x)} \left[ 
\log P_{X_\mathcal{D}} (x|z)\right] - \beta  D_{KL}\left( P_Z(\cdot|x) || P_{X_\mathcal{D}(\cdot)}  \right)
\end{equation*}

A VAE trained with this modified ELBO is known as a $\beta$-VAE. When $\beta>1$, the KL divergence term gains importance and it encourages the model to learn a more effective data representation, leading to some disentanglement of the data inherent degrees of freedom. Conversely, setting $0<\beta<1$ prioritizes reconstruction fidelity, which, however, comes at the expense of generalization and the quality of the data representation in the latent space.

\null

In a Mixture of VAEs, the distribution $P_X$ of samples from the data manifold is approximated by a combination of $q$ random variables $X_k$, for $k=1,\dots,K$. This is expressed as
$$P_X (x) = \sum_{i=1}^K \alpha_i P_{X_i}(x)$$
where $\alpha_i >0$ are mixing weights that satisfy $\sum_{i=1}^K \alpha_i = 1$. Each random variable $X_k$ within this mixture model is learned by a distinct VAE.
The atlas of an embedded manifold $M$ can be learned by representing it as a mixture model of $\beta$--VAEs. In this framework, the encoders $\mathcal{M}^{\mathcal{E},k}$ serve as the local charts of the manifold, while the decoders $\mathcal{M}^{\mathcal{D},k}$ act as its local parametrizations (which are the inverse maps of the local charts), for $k=1,\dots,q$.
A mixture of VAEs is trained by maximizing the ELBO of the mixture. This process simultaneously learns both the weights of the individual VAEs and the mixing weights. For a detailed explanation, refer to \cite{ManifoldVAEs}.

\subsubsection{Invertible Neural Netowrks} \label{subsec:inns}
As the maps $\phi$ in \Cref{def:manifold_charts} are injective, any neural network used to realize them must be one-to-one. To ensure this property, we employ an architecture that utilizes the coupling blocks proposed in \cite{Dinh17, ardizzone19}, thereby constructing an invertible neural network (INN).
The fundamental unit of the INN is a reversible block, consisting of two complementary affine coupling layers. Let $x \in \mathbb{R}^n$, where $n \in \mathbb{N}$, be an input datum. Given natural numbers $n_1, n_2$ such that $n_1+n_2=n$, we factorize $\mathbb{R}^n$ as $\mathbb{R}^{n_1}\times\mathbb{R}^{n_2}$. We then consider the projections of $x$, denoted as $(x_1,x_2) \in \mathbb{R}^{n_1} \times \mathbb{R}^{n_2}$. Next, we introduce the map:
\begin{equation*}
\begin{split}
F :\mathbb{R}^{n_1}\times\mathbb{R}^{n_2} & \rightarrow \mathbb{R}^{n_1}\times\mathbb{R}^{n_2}\\
(x_1,x_2) & \mapsto (y_1,y_2)
\end{split}
\end{equation*}
defined by
\begin{equation}
\begin{cases}
y_1 = x_1 e^{C(x_2)}+D(x_2)\\
y_2 = x_2 e^{A(x_1)}+B(x_1),
\end{cases}
\end{equation}
with $A,B : \mathbb{R}^{n_1} \rightarrow \mathbb{R}^{n_2}$ and $C,D : \mathbb{R}^{n_2} \rightarrow \mathbb{R}^{n_1}$ four not necessarily invertible sub--networks. The map~$F$ has the trivial analytic inverse
\begin{equation}
\begin{cases}
x_2 = \left( y_2 - B(y_1) \right)  e^{-A(y_1)}\\
x_1 = \left( y_1 - D(z_2) \right)  e^{-C(z_2)}
\end{cases}.
\end{equation}
Invertible networks allow to jointly learn both the forward and backward functions, meaning the decoder can be directly obtained as the inverse of the encoder. Furthermore, unlike a standard non-invertible VAE, an invertible VAE is always guaranteed to use the full latent space because it is injective.
This characteristic can be used, for example, to guarantee the invertibility of local charts, as demonstrated in~\cite{ManifoldVAEs}, especially when the intrinsic dimension of the data manifold is known. On the other hand, an invertible VAE cannot learn a data representation using only a subspace of the latent space. This limitation can be problematic if the intrinsic dimension is not known beforehand. 

\section{Intrinsic dimension estimate in VAEs via geometric notions }\label{section:dimension_estimate_general}
Suppose we know the metric $h$ of the ambient space where the data is embedded -- in applications, typically a Euclidean space with standard metric. Once a VAE has been trained on a dataset with a latent space of sufficiently large dimension, we can use the pullback of $h$ through the decoder $\mathcal{D}$ of the network, as per \Cref{def:pullback}, namely $g = J_{\mathcal{D}}^T J_{\mathcal{D}}$, to estimate the dimension of the Riemannian manifold embedded in the latent space. 
Indeed, if the dimension of the VAE latent space is greater than the intrinsic dimension of the data manifold, $g$ is a degenerate metric. In such cases, we can estimate the data manifold dimension using the numerical rank of the metric $g$, as illustrated in \Cref{ex:degenerate_metric}. 

\null

\noindent Practically, for a sample of points in the dataset, we compute the eigenvalues of the metric $g$ and we order them in decreasing order. An eigenvalue $\lambda_i$ is considered effectively ``null'' (i.e., $\lambda_i \ll 1$) if there is a large gap (orders of magnitude) between its value and $\lambda_{i-1}$. An example of this situation is shown in \Cref{fig:dimension_estimate_simple}. The intrinsic dimension of the data manifold is then determined by the average number of non-null eigenvalues observed across the sampled points in the dataset. To avoid inspecting the eigenvalue plot for each point of the sample, we can consider the eigenvalues averaged over all elements of the sample, as shown in \Cref{algo:ID_estimation}.
\begin{algorithm}[H]
\begin{algorithmic}[1]
\small
\State {Given a dataset with $N$ samples and a VAE trained on the dataset, with decoder $\mathcal{D}$:}
\For {$k=0,1,\dots,N-1$}
\State {compute $J_{\mathcal{D}}(x_k)$}
\State {compute the pullback metric $g(x_k) = J_\mathcal{D}^T J_\mathcal{D}$ }
\State {diagonalize $g(x_k)$ to obtain its eigenvalues $\underline{\lambda_k}$}
\EndFor
\State{compute the element-wise mean of $\{ \underline{\lambda}_k \}_{k=0,\cdots,N-1}$}
\State{search for the first ``null'' eigenvalue}
\end{algorithmic}
\caption{\small ID estimation algorithm.
}
\label{algo:ID_estimation}
\end{algorithm}
\noindent
Notice that, provided the network has been trained to adequately reproduce the data, either the decoder or the encoder can be used to compute the pullback metric, as the following proposition holds.
\begin{proposition}
Let $\mathcal{E} : \mathbb{R}^n \rightarrow \mathbb{R}^m$ and $\mathcal{D} : \mathbb{R}^m \rightarrow \mathbb{R}^n$ be two differentiable maps with $\mathcal{D} \mathcal{E} = \mathbb{I}$ and suppose that $n>m$. Then $J_\mathcal{D}^T J_\mathcal{D} = \left( J_\mathcal{E} J_\mathcal{E}^T \right)^{^\dagger}$.
\end{proposition}
\begin{proof}
Since $\mathcal{D} \mathcal{E} = \mathbb{I}$, then $J_\mathcal{D} J_\mathcal{E} = \mathbb{I}$ and $J_\mathcal{E}^T J_\mathcal{D}^T = \mathbb{I}$. Therefore we can write
\begin{equation*}
J_\mathcal{E}^T J_\mathcal{D}^T J_\mathcal{D} J_\mathcal{E}  = \mathbb{I}
\end{equation*}
Multiplying to the left by the pseudoinverse $J_\mathcal{E}^{T \dagger}$  and to the right by $J_\mathcal{E}^{\dagger}$ then gives the thesis
\begin{equation*}
J_\mathcal{D}^T J_\mathcal{D} = J_\mathcal{E}^{T \dagger} J_\mathcal{E}^\dagger = (J_\mathcal{E} J_\mathcal{E}^T )^{\dagger}
\end{equation*}
\end{proof}
From a computational standpoint, a significant implication of this proposition is that calculating $J_\mathcal{E}$ might require less memory compared to $J_\mathcal{D}$. This depends on the specific VAE structure and the automatic differentiation algorithm used for computing the Jacobian matrices.
It is also worth noting that, in practice, $J_\mathcal{E}$ is, strictly speaking, always full rank since its eigenvalues are never truly zero, but rather very small. Consequently, $J_\mathcal{E} J_\mathcal{E}^T$ is also a full rank matrix. To avoid numerical instability when computing $(J_\mathcal{E} J_\mathcal{E}^T )^{\dagger}$, we can instead determine its rank by calculating the numerical rank of $J_\mathcal{E} J_\mathcal{E}^T$. Given these considerations, we employ the matrix $J_\mathcal{E} J_\mathcal{E}^T$ in \Cref{algo:ID_estimation} instead of $g = J_\mathcal{D}^T J_\mathcal{D}$.
\begin{remark}
We have observed that, due to numerical errors, it may happen that the smallest eigenvalues turn out to be negative.
This is merely due to the floating point precision used in the computations (usually, 32 bits) and, in any case, ``negative'' eigenvalues are to be considered ``null''. 
\end{remark}

\begin{example}\label{ex:ID_estimate}
We illustrate the proposed method for estimating the ID considering two simple manifolds, a circle and a paraboloid. 
We use in both cases $\beta$-VAE architectures with a latent space of dimension $3$, setting $\beta=1$ and $\beta=0.1$, respectively. The VAEs are trained to learn in the first case a representation of a unit circle centered at the origin and in the second case the paraboloid $z=x^2+y^2$ with $x,y \in [-1,1]$. 
The results are presented in \Cref{fig:dimension_estimate_simple}.
We apply this technique to a more complex inverse problem in \Cref{sec:Numerical_experiments}. 

\begin{figure}[h!]
\centering
\begin{subfigure}[t]{0.5\textwidth}
		\centering
		\includegraphics[width=.95\linewidth]{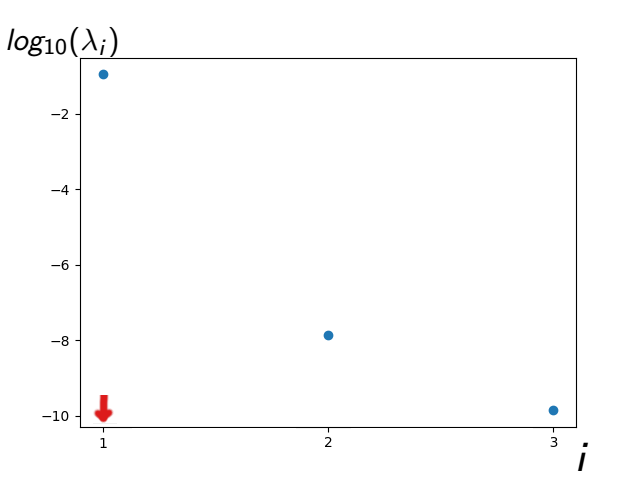}
		\caption{Eigenvalues of the pullback metric for the circle}
		\label{fig:dimension_estimate_ex1}
	\end{subfigure}%
	\begin{subfigure}[t]{0.5\textwidth}
		\centering
		\includegraphics[width=0.95\linewidth]{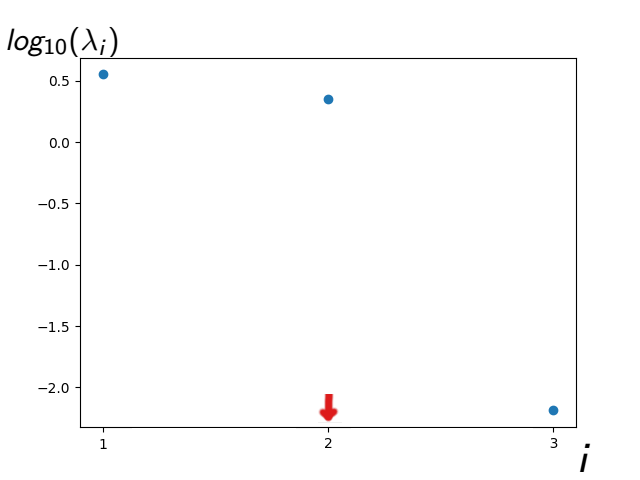}
		\caption{Eigenvalues of the pullback metric for the paraboloid.}
		\label{fig:dimension_estimate_ex2}
	\end{subfigure}
\caption{Plots of the eigenvalues of the pullback metric for the circle (a) and the paraboloid (b). In panel (a), the second eigenvalue is significantly smaller (by a factor of $10^7$) than the first. Consequently, our method estimates the correct ID of $1$ for the circle. For the paraboloid in panel (b), the estimated ID is the correct value $2$, evident from the large gap observed between the second and third eigenvalues.}
\label{fig:dimension_estimate_simple}
\end{figure}
\end{example}
We emphasize that the dimension estimated by our method is contingent upon how well the VAE is trained to reproduce the original manifold. In particular, networks that yield poor reconstructions could provide an unreliable estimate. For instance, employing high values 
of~$\beta$ in a $\beta$-VAE or applying Lasso regularization too aggressively can result in an underestimate of the manifold's dimension, as illustrated in \Cref{fig:bad_estimate}.
\begin{figure}[h!]
\centering
\begin{subfigure}[t]{0.5\textwidth}
		\centering
		\includegraphics[width=1.\linewidth]{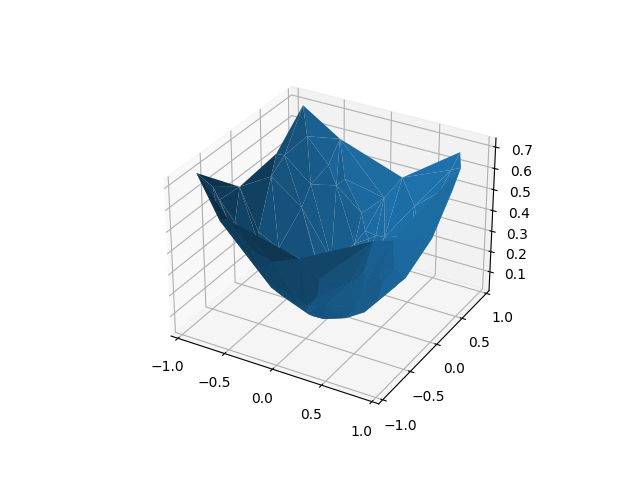}
		\caption{The adequate value of $\beta=0.1$ \\ yields a good reconstruction of the paraboloid.}
		\label{fig:bad_estimate_ex1}
	\end{subfigure}%
	\begin{subfigure}[t]{0.5\textwidth}
		\centering
		\includegraphics[width=1.\linewidth]{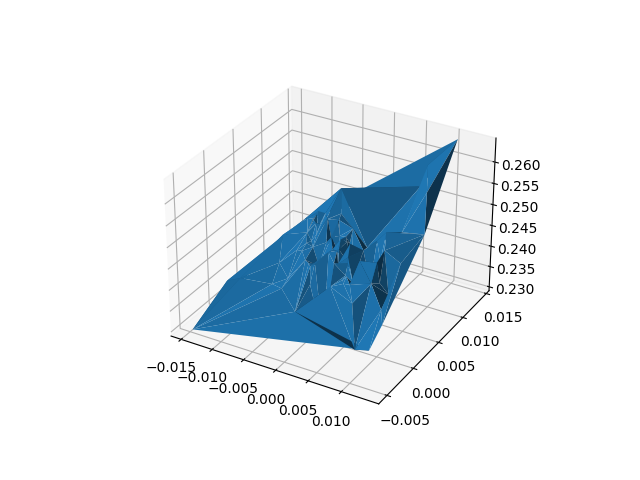}
		\caption{
        \noindent
Using $\beta=10$ yields a very poor reconstruction of the paraboloid, where only the region near the vertex (zoomed in the panel) is learned, the result outside being completely useless. It is even unclear if the manifold represented by the VAE is a surface or a complex $3D$ shape.}
		\label{fig:bad_estimate_ex2}
	\end{subfigure}
	\begin{subfigure}[t]{0.5\textwidth}
		\centering
		\includegraphics[width=0.95\linewidth]{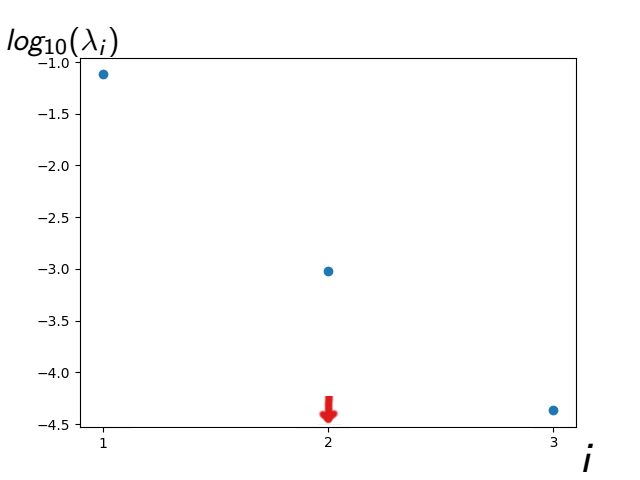}
		\caption{Eigenvalues of the pullback metric for the poor reconstruction of the paraboloid shown in panel (b).}
		\label{fig:bad_estimate_ex3}
	\end{subfigure}
\caption{Reconstruction of the paraboloid $z=x^2+y^2$ by two different VAEs obtained setting $\beta=0.1$ (panel (a)) and $\beta=10$ (panel (b)),
respectively. Panel c): Plot of the eigenvalues of the pullback metric for the $\beta$-VAE of case (b). 
Compared to \Cref{fig:dimension_estimate_ex2} -- corresponding to the VAE which gives the reconstruction in (a) -- it is not apparent a gap between the eigenvalues.}
\label{fig:bad_estimate}
\end{figure}

\section{Application to the inverse problem in CT imaging}
\label{section:application_to_inverse_problem}

Computed tomography (CT) is a prototypical inverse problem where the goal is to reconstruct an image of an object’s internal structure from a series of external X-ray projections which correspond to line integrals of the object's attenuation coefficients. Namely, given a 2D slide of a 3D object, the line--averaged 
X-ray attenuation $u(x)$ along the projection rays is given by the Radon transform
$$
Ru(\theta,t)=\int_{\ell}u(x)\, dx
$$
where $\ell=\{ x=(x_1,x_2) \, | \, 
(x_1 \cos(\theta)+x_2\sin\theta)=t\}$, $ t \in [-1,1],
\theta \in [0,\pi)$, is the projection ray.
The collection of the projections for several angles
is called sinogram. One could reconstruct
the field $u(x)$ from the sinogram by inverting
the Radon transform, i.e. computing $R^{-1}$. 
In practice, however, one disposes only of a discrete and often noisy set of projections taken at limited angles. Due to these issues, the inversion process is typically ill-posed and ill-conditioned, requiring regularization or prior information to obtain stable and accurate reconstructions. 
Deep learning approaches have the potential to significantly advance CT reconstruction (see, e.g., \cite{Wang2020,Koetzier2023}). Neural networks can be employed in various ways: as post-processing denoisers, end-to-end mappings from sinograms to images or can be embedded within iterative schemes (e.g., learned unrolling). These data-driven techniques leverage large datasets to implicitly learn image priors, leading to reconstructions with improved artifact suppression, noise robustness, and resolution. 

\null

\noindent In this work we do not aim at proposing
a new workflow for CT reconstruction, but we take 
this problem as a paradigm of inverse problem to explore the benefits of disposing of a computational technique for ID estimation. 
Namely, we consider the neural network model introduced in~\cite{Benfenati2024}. 
In this model,
the observations (here, the sinograms)
and the generating physical field (here,
the absorption coefficient represented as a pixelized field) are treated along two
separated autoencoder pathways, which 
are bridged via a third network 
that learns the image-to-observation relation (here, the Radon transform) in the respective latent spaces.
More in detail, we consider the following subnetworks relevant to our discussion (see also \Cref{fig:full_networks} and refer to the Appendix):
\begin{itemize}
\item $\beta$-VAEs (``Whitney embedders'') $\mathcal{W}_\mathcal{S}$ and $\mathcal{W}_\mathcal{I}$  composed
of encoder $\mathcal{E}_\mathcal{S}$ ($\mathcal{E}_\mathcal{I}$) and decoder $\mathcal{D}_\mathcal{S}$ ($\mathcal{D}_\mathcal{I}$). They perform a first gross dimensional reduction of the input images and sinograms, respectively, allowing to significantly reduce the working dimension  
\item  mixtures of VAEs $\mathcal{M}_\mathcal{I}
^{i}$ and $\mathcal{M}_\mathcal{S}^{i}$, for $i=1,
\dots K$, which build a collection of local charts for the embedding space of the images (right subnetwork in the inset of  \Cref{fig:full_networks}) and the
embedding space of the sinograms (left subnetwork in the inset of \Cref{fig:full_networks}). 
The number of charts $K$ is a hyperparameter, whilst ID estimation of the sinogram and image manifolds 
is performed across the two Whitney embedders
using the metric as described above 
\end{itemize} 
In addition, fully connected networks $\widetilde{\mathcal{R}}$
and $\widetilde{\mathcal{R}}^{-1} $ perform the data-driven Radon transform across the latent spaces produced by the encoders $\mathcal{E}_\mathcal{I}$ and 
$\mathcal{E}_\textbf{S}$. Specifically, given a $d \times d$ image $x$ and
$d \times d$ sinogram $y$,
we define 
$$\mathcal{E}_\mathcal{S} (\mathcal{R}(x)) \approx \widetilde{\mathcal{R}} (\mathcal{E}_\mathcal{I}),
\qquad 
\mathcal{E}_\mathcal{I} (\mathcal{R}^{-1}(y)) \approx \widetilde{\mathcal{R}}^{-1} (\mathcal{E}_\mathcal{S}(y))$$
Notice that
we do not build the networks $\widetilde{\mathcal{R}}$ and $\widetilde{\mathcal{R}}^{-1}$ directly between the local charts. This is because, in general, data points belonging to image chart $U_I$ are not necessarily mapped by the Radon transform exclusively to the same sinogram chart $U_S$. 
Instead, different image points in the same chart may correspond to different sinogram charts. Actually, this is
a general behavior in differential geometry, not related to the
specific net architecture. 

\begin{figure}
\centering
\includegraphics[scale=.94,angle=0]{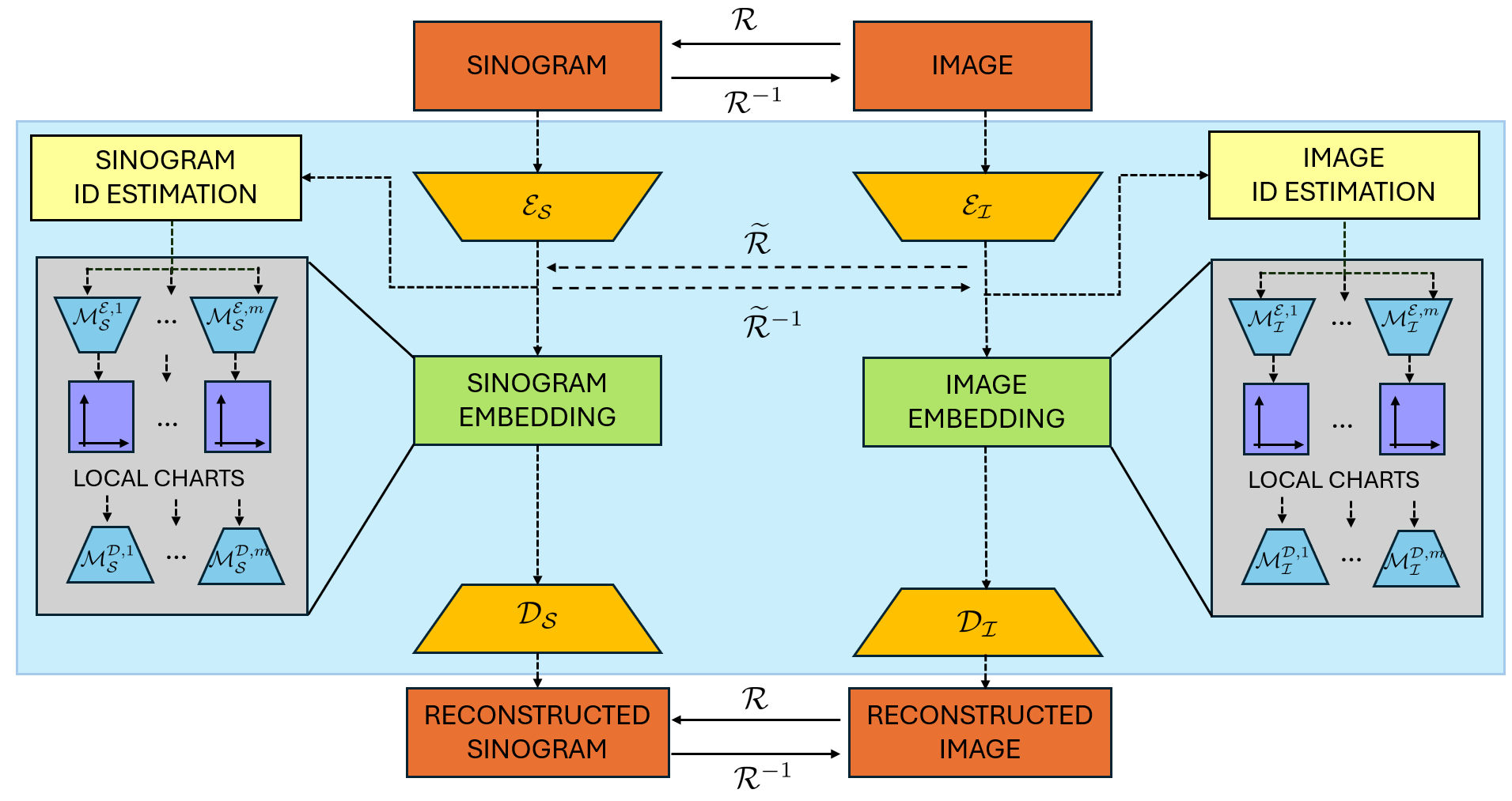}
\caption{Architecture of the network for CT reconstruction. First a ``Whitney embedding'' is performed by the encoders  $\mathcal{E}_\mathcal{S}$ and $\mathcal{E}_\mathcal{I}$ for the sinograms and the images, respectively. The two sub--networks in the insets at the left and right are employed to build the local charts. The left network produces the local charts for the sinograms by means of the encoder $\mathcal{M}_\mathcal{S}^\mathcal{E}$ and the decoder $\mathcal{M}_\mathcal{S}^\mathcal{D}$,
while the encoder $\mathcal{M}_\mathcal{I}^\mathcal{E}$ and the decoder $\mathcal{M}_\mathcal{I}^\mathcal{D}$ in the right part builds the atlas for the sinograms. The fully connected networks $\widetilde{\mathcal{R}}$ and $\widetilde{\mathcal{R}}^{-1}$ realize the push--forward of the Radon and inverse Radon transforms maps $\mathcal{R}$ and $\mathcal{R}^{-1}$.  Data flow along the pathways signaled by the arrows.}
\label{fig:full_networks}
\end{figure}



\null

 \noindent Training is performed separately
 for the two embedding networks $\mathcal{W}_I$ and $\mathcal{W}_S$. Namely, we minimize a modification of the $\beta$-VAE loss \cite{higgins2017betavae} that includes Lasso regularization to encourage the network to select the most important features \cite{Allerbo2021NonlinearSD,JMLR:v22:20-848}. Let $\hat{x}$ be the reconstruction of the network and $p(z|x)$ be the distribution of the posterior. The loss function is
 defined as\begin{equation}\label{eq:embedding_loss}
\mathcal{L}(x,\hat{x}) = \dfrac{\alpha}{2} \ T(\hat{x},x) + \beta \ \text{KL}(p(z|x) \ | \ \mathcal{N}(0,1) ) + \gamma \sum_{\theta \in \Theta} \| \theta \|_1
\end{equation}
where $\Theta$ denotes the weights of the network, $\gamma$ is the parameter of the Lasso regularization and $T(\cdot,\cdot)$ is the Tanimoto distance \cite{Lipkus1999APO}, a dissimilarity coefficient widely employed in the screening of molecules in chemistry \cite{df439fc78ddb498d9eb828a8974212dd}, in the analysis of similarity of video sequences \cite{Tashlinskii2017}, and in image classification \cite{PRAVEENKUMAR2022100063} and segmentation \cite{SPIE2014}. 
The reason behind the choice of the Tanimoto distance as the reconstruction error in \Cref{eq:embedding_loss} is that it is  more sensitive to small changes compared to the Euclidean one \cite{TanimotoPres}, a useful feature in the presence of thin lines as in 
our dataset. 
Namely, given $x,y \in \mathbb{R}^n$, $n \in \mathbb{N}$, the Tanimoto distance is defined as:
\begin{equation*}
T(x,y) = 1 - \dfrac{x \cdot y}{\|x\|^2+\|y\|^2-x \cdot y}
\end{equation*}
A brief computation reveals that for $D \times D$ images, the Euclidean distance between two grayscale pictures $x$ and $y$ can be bounded by the following expression (Tanimoto coefficient):
\begin{equation*}
\| x - y \|^2_2 \leq 2 D^2 \dfrac{T(x,y)}{1-T(x,y)}
\end{equation*}
Therefore, for $\alpha > 2d^2$, the loss function as per \Cref{eq:embedding_loss} (without the Lasso regularization term) bounds the standard $\beta$-VAE loss function defined using the $L^2$ Euclidean loss. 
\section{Numerical experiments}\label{sec:Numerical_experiments}
\subsection{Dataset}

The dataset we use in our experiments is a simplified version of the COULE (Contrasted Overlapping Uniform Lines and Ellipses) dataset~\cite{COULE}, a collection of images designed to train and test neural networks for CT applications in medical imaging. Each sample is a $128 \times 128$ gray-scale image containing two flat circles of different colors and a white line of constant thickness, as shown in \Cref{fig:embedding_reconstructions_a}. The corresponding
sinograms are shown in \Cref{fig:embedding_reconstructions_c}. For this dataset images have ID$=12$ since they are fully described by a minimum of $12$ coordinates (positions of the centers, radii, and color of the two circles; position of the midpoint, inclination, and color of the line). 
We train the networks $\mathcal{W}_I$ and $\mathcal{W}_S$ for $1000$ epochs, setting $\alpha = 2D^2$, $\beta = 25$, and $\gamma = 0.0001$ in \Cref{eq:embedding_loss}. The dimension of the embedding space is $25$. The dataset is first normalized to the range $[0,1]^{D^2}$ and then, to avoid overfitting, a slight jitter filter is applied to the data. In \Cref{fig:embedding_reconstructions}(b,d) 
we show the reconstructions of the images and sinograms obtained by $\mathcal{W}_\mathcal{I}$ and $\mathcal{W}_\mathcal{S}$, respectively. 

\begin{figure}[h!]
\begin{center}
	\begin{subfigure}[t]{0.5\textwidth}
		\centering
		\includegraphics[width=.95\linewidth]{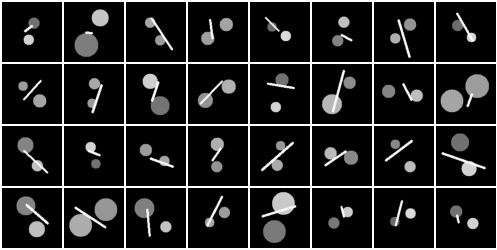}
		\caption{}
		\label{fig:embedding_reconstructions_a}
	\end{subfigure}%
	\begin{subfigure}[t]{0.5\textwidth}
		\centering
		\includegraphics[width=0.95\linewidth]{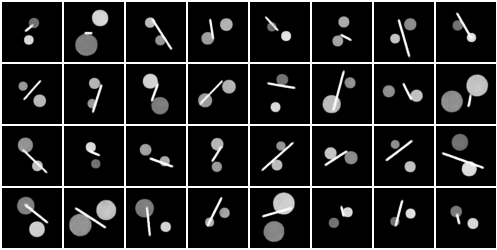}
		\caption{}
		\label{fig:embedding_reconstructions_b}
	\end{subfigure}
	
	\begin{subfigure}[t]{0.5\textwidth}
		\centering
		\includegraphics[width=.95\linewidth]{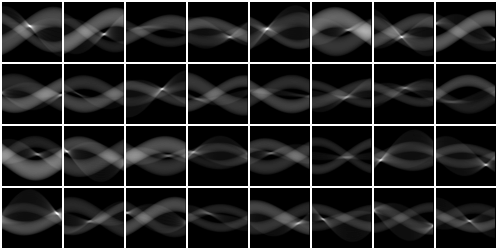}
		\caption{}
		\label{fig:embedding_reconstructions_c}
	\end{subfigure}%
	\begin{subfigure}[t]{0.5\textwidth}
		\centering
		\includegraphics[width=0.95\linewidth]{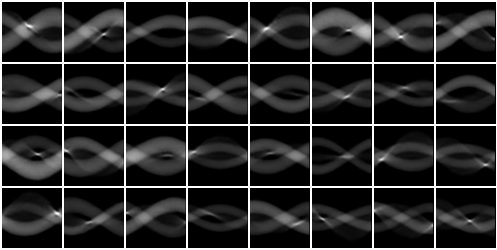}
		\caption{}
		\label{fig:embedding_reconstructions_d}
	\end{subfigure}
\caption{Top panel: Input images (a) and  corresponding reconstructions obtained by the embedding network $\mathcal{W}_\mathcal{I}$ (b). Bottom panel: Input sinograms (c) and corresponding reconstructions obtained by the embedding network $\mathcal{W}_\mathcal{S}$ (d).}
\label{fig:embedding_reconstructions}
\end{center}
\end{figure}

In \Cref{fig:radontilde} we show for a sample in the dataset (left) the reconstruction (right) obtained from the transformation $\mathcal{D}_\mathcal{I} \circ \widetilde{\mathcal{R}}^{-1} \circ \widetilde{\mathcal{R}} \circ \mathcal{E}_\mathcal{I}$. 
\begin{figure}[h!]
\centering
\includegraphics[scale=0.23]{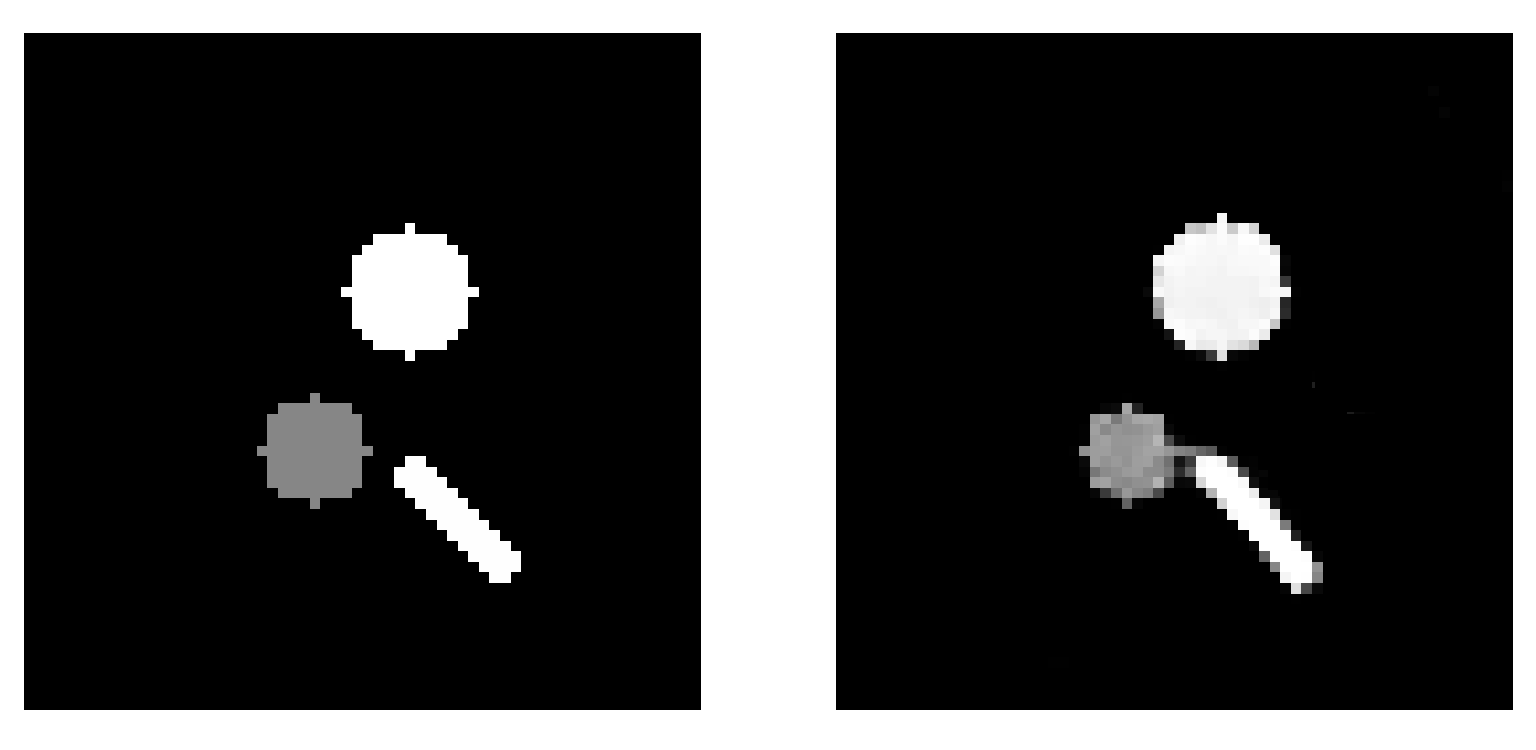}
\caption{Original image (left) and reconstruction (right) obtained from the transformation $\mathcal{D}_I \circ \widetilde{\mathcal{R}}^{-1} \circ \widetilde{\mathcal{R}} \circ \mathcal{E}_I$ for a sample in 
the COULE dataset.}
\label{fig:radontilde}
\end{figure}

The code for the numerical experiments can be found at \url{https://github.com/alessiomarta/Intrinsic_Data_Structures_via_Singular_Metrics_in_VAEs}.

\subsection{Estimate of the dataset ID}\label{subsec:dimension_estimate}
Before training the mixture of $\beta$-VAEs to construct the local charts, we first  estimate the intrinsic dimension of the data manifold. To do this, we follow the procedure outlined in \Cref{section:dimension_estimate_general}, which involves computing and then studying the eigenvalues of the singular metric $g = J_\mathcal{E} J_\mathcal{E}^T$. Our analysis reveals that the image manifold is of dimension $12$ (matching the correct dimension of our dataset), while the sinogram manifold is of dimension~$11$, as depicted in \Cref{fig:dimension_estimate}.
We observe that the networks $\mathcal{W}_I$ and $\mathcal{W}_S$ produce data representations with differing dimensions. This is presumably a combined consequence of the 
physics of the problem, the discrete number of measurements in the sinogram and 
the learning process of the networks.
We also compare the ID estimation of our method with other algorithms implemented in the {\tt scikit-dimension} package~\cite{scikit_dimension}. The results, reported in \Cref{tab:cfr_id}, show that lPCA overestimates the ID for both datasets, providing a reasonable approximation for the sinograms but a very poor estimate for the images. CorrID, on the other hand, consistently underestimates the ID for both datasets. MLE provides a good ID estimate for the image dataset but a lower ID for the sinogram dataset. Our method accurately predicts the correct ID for the image dataset and provides a good approximation for the sinogram dataset.

\begin{figure}[h]
\centering
\begin{subfigure}[t]{0.45\textwidth}
		\centering
		\includegraphics[width=.95\linewidth]{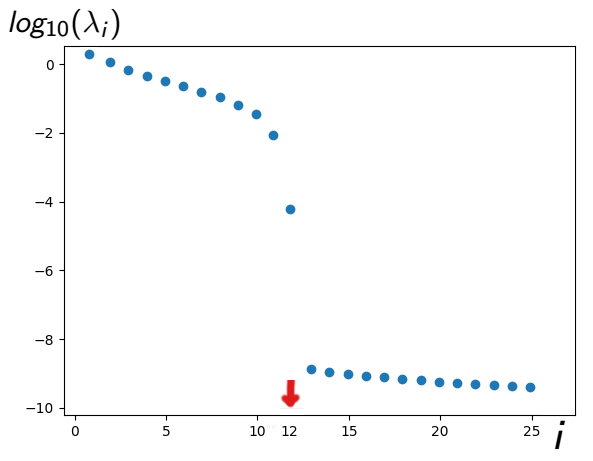}
		\label{fig:dimension_estimate_img}
	\end{subfigure}%
	\begin{subfigure}[t]{0.45\textwidth}
		\centering
		\includegraphics[width=0.95\linewidth]{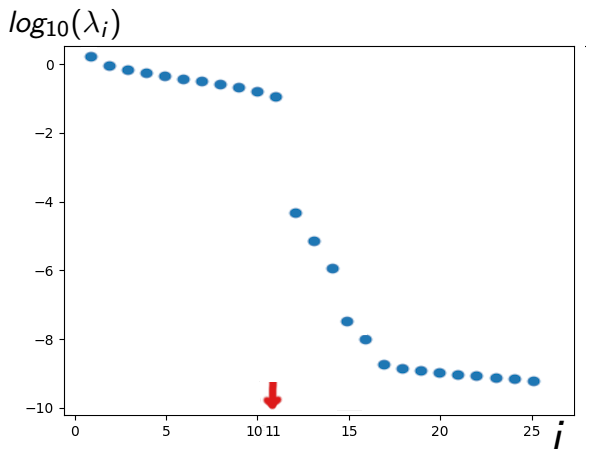}
		\label{fig:dimension_estimate_sin}
	\end{subfigure}
\caption{Plot of the eigenvalues of the pullback metric for an image (left) and its corresponding sinogram (right) from the COULE dataset. 
The red arrow indicates the dimension at which a significant gap in the eigenvalues occurs. In the left panel, this gap appears at the $12^{th}$ eigenvalue, leading to an estimated ID of~$12$, which is the correct dimension of the dataset. The  estimated ID of the sinogram is $11$.}
\label{fig:dimension_estimate}
\end{figure}

\begin{table}[h]
    \centering
    \begin{tabular}{|c|c|c|}
    \hline 
    ID Estimation method & Images & Sinograms \\ 
    \hline 
    lPCA & 29 & 13  \\ 
    \hline 
    MLE & 13 & 8 \\ 
    \hline 
    CorrID & 8 & 5 \\ 
    \hline 
    Ours & 12 & 11 \\ 
    \hline 
    \end{tabular} 
    \caption{Comparison of the estimated ID for image and sinogram datasets, using our proposed method vs local PCA (lPCA), Maximum Likelihood Estimator (MLE), and the correlation dimension method (CorrID). The estimates were obtained using implementations of these algorithms from \cite{scikit_dimension} and for all algorithms, the ID estimation was performed using a sample of 500 elements.}
    \label{tab:cfr_id}
\end{table}

\subsection{Local charts with mixtures of VAEs}

\noindent
As already discussed in above, there is generally no definitive guidance on the optimal number of charts required to construct an atlas. In practical applications, this choice should represent a trade-off between computational efficiency and model flexibility. While using too few charts may lead to a poor manifold representation, employing too many charts is computationally expensive.
In our specific case, the data manifold for our simplified COULE image dataset is diffeomorphic to $S^1 \times S^1 \times (0,1)^{10}$. Thus, we know \textit{a priori} that the minimum number of local charts is $4$. We consequently used $4$ networks in the mixture of VAEs, with the manifold dimension set to $12$, as estimated in the previous section. However, for the corresponding sinograms, we lack theoretical suggestions on the number of charts, as transforms generally do not preserve the number of local charts. We again chose to use $4$ charts, which yielded good results in our experiments. Should the mixture of VAEs exhibit excessively high loss (indicating poor reconstruction quality of the full network), it may be beneficial to experiment with different numbers of local charts.
Fig.\ref{fig:latent_reconstructions} shows an example of the local charts for the images and sinograms and a few random samples generated by the chart latent spaces. These atlases can then be used to solve a CT imaging inverse problem employing, for example, the Riemannian gradient descent scheme proposed in \cite{ManifoldVAEs}.

\begin{figure}[t]
\begin{center}
	\begin{subfigure}[t]{0.35\textwidth}
		\centering
		\includegraphics[width=.95\linewidth]{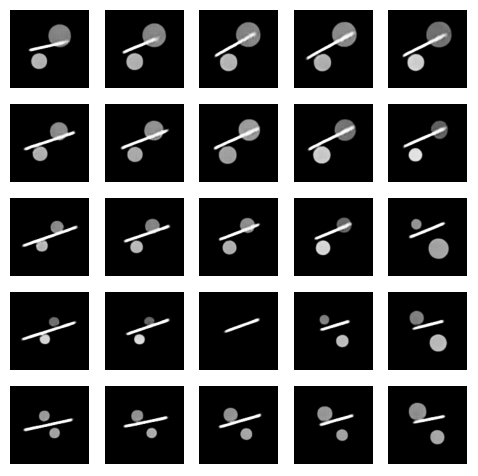}
		\caption{}
		\label{fig:local_chart_img}
	\end{subfigure}%
	\begin{subfigure}[t]{0.35\textwidth}
		\centering
		\includegraphics[width=0.95\linewidth]{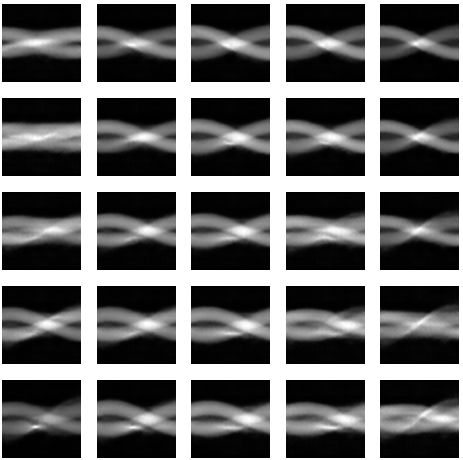}
		\caption{}
		\label{fig:local_chart_sin}
	\end{subfigure}
	
	\begin{subfigure}[t]{0.35\textwidth}
		\centering
		\includegraphics[width=.95\linewidth]{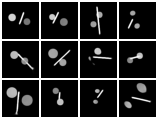}
		\caption{}
		\label{fig:samples_img}
	\end{subfigure}%
	\begin{subfigure}[t]{0.35\textwidth}
		\centering
		\includegraphics[width=0.95\linewidth]{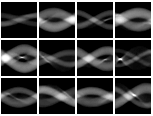}
		\caption{}
		\label{fig:samples_sin}
	\end{subfigure}
\caption{Local charts for the images (a) and sinograms (b). Starting from the left-bottom corner, in both cases the different pictures correspond to the
projection on the first two coordinates of a local chart $(0.2 \cdot j, 0.2 \cdot k)$ with $j,k = 0,\ldots,4$. Random samples generated in the local charts latent space for images (c) and sinograms (d).}
\label{fig:latent_reconstructions}
\end{center}
\end{figure}

\newpage
\clearpage

\subsection{Network pruning on the learned manifold}

A visual inspection of the outputs from the last convolutional blocks of the encoder $\mathcal{E}_I$ and the first block of the decoder $\mathcal{D}_I$ reveals that many learned filters exhibit redundancy highlighting similar features or convey no meaningful information, as shown in the top panel of \Cref{fig:weights_analysis}. Redundancy of feature maps in convolutional layers is a well-known phenomenon associated with an inefficient capacity utilization of the network \cite{Hao2021ManipulatingIF,bdcc7040159,LU2024104547}. It is worth noting that in the case under analysis this can be due to the simplicity of the dataset. For example in \Cref{fig:last_encoder} we can see that some filters learn to detect the presence of an entire segment. In contrast, with more complex datasets, a single object in an image may be broken up in more basic shapes, with different filters detecting the presence of different features in the same object \cite{9607765}. The redundancy of the filters is also reflected by the large number of weights with a small $L^1$ norm (\Cref{fig:l1_encoder}) within the fully connected layers that receive the output of the encoder's last convolutional layer (refer to \Cref{fig:embedding_architecture}). This pattern holds true for the fully connected layers of decoder $\mathcal{D}_I$ as well (\Cref{fig:l1_decoder}). Similar results were observed from the inspection of the subnetworks processing the sinograms, $\mathcal{E}_S$ and $\mathcal{D}_S$. This analysis suggests that a significant number of weights can be pruned without detriment to reconstruction quality. Crucially, the practical limit for pruning can be identified by a noticeable change in the estimated intrinsic dimension, as this indicates a degraded representation of the data manifold.
We pruned the weights of the fully connected blocks in both the encoder $\mathcal{E}_I$ and the decoder $\mathcal{D}_I$ using PyTorch's \texttt{prune} module, performing a global unstructured pruning that removes weights with the lowest $L^1$ norm from each layer. The fraction of parameters pruned was varied from $p = 0$ to $p = 0.999$ in steps of $0.001$. We first observed a degradation in reconstruction quality beyond a pruning threshold of $p = 0.996$, which corresponds to retaining only approximately $3 \cdot 10^4$ of the most important weights (as detailed in \Cref{table:FC_weights} and \Cref{fig:rec_loss_pruning}).

\begin{table}[bht]
\centering
\begin{tabular}{|c|c|c|c|}
\hline 
Layer & \makecell{No. of weights \\ (original model)} & \makecell{No. of weights \\ (pruning ratio 0.996)} & \makecell{No. of weights \\ (pruning ratio 0.999)}\\ 
\hline 
1st FC layer of $\mathcal{E}_I$ & 4194560 & 16772 & 4194 \\ 
\hline 
2nd FC layer of $\mathcal{E}_I$ & 32896 & 129 & 35\\
\hline 
1st FC layer of $\mathcal{D}_I$ & 4352 & 35 & 13\\
\hline 
2nd FC layer of $\mathcal{D}_I$ & 33024 & 139 & 33 \\
\hline 
3rd FC layer of $\mathcal{D}_I$ & 4210688 & 16852 & 4213 \\
\hline 
\end{tabular}
\caption{Number of weights of the fully connected layers of $\mathcal{E}_I$ and $\mathcal{D}_I$ in the architecture of \Cref{fig:embedding_architecture} before and after pruning. The distribution of the $L^1$ norms of the weights of the first and the last layers before the pruning are reported in \Cref{fig:l1_encoder,fig:l1_decoder}, respectively. 
} 
\label{table:FC_weights}
\end{table}

\begin{figure}
  \begin{subfigure}[c]{.5\linewidth}
    \centering
    \includegraphics[width=\linewidth]{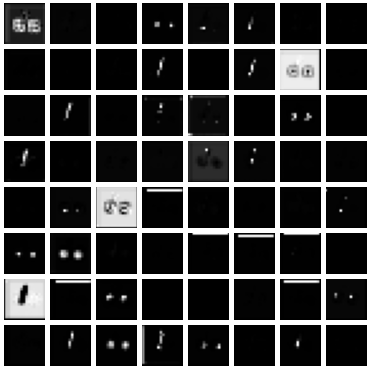}%
    \caption{}
    \label{fig:last_encoder}
  \end{subfigure}\hfill
  \begin{tabular}[c]{@{}c@{}}
    \begin{subfigure}[c]{.48\linewidth}
      \centering
      \includegraphics[width=\linewidth]{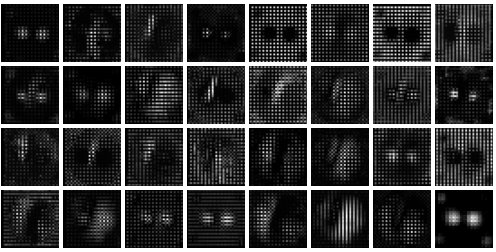}%
      \caption{}
      \label{fig:first_decoder}
    \end{subfigure}\\
    \noalign{\bigskip}%
    \begin{subfigure}[c]{.19\linewidth}
      \centering
      \includegraphics[width=\linewidth,page=2]{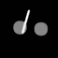}%
      \caption{}
          \label{fig:original_img}
    \end{subfigure}
  \end{tabular}
  
    \begin{subfigure}[c]{.5\linewidth}
    \centering
    \includegraphics[width=\linewidth]{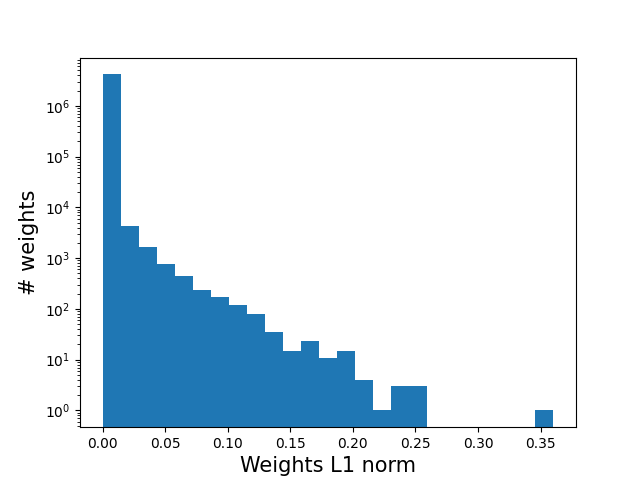}%
    \caption{}
    \label{fig:l1_encoder}
  \end{subfigure}
    \begin{subfigure}[c]{.5\linewidth}
    \centering
    \includegraphics[width=\linewidth]{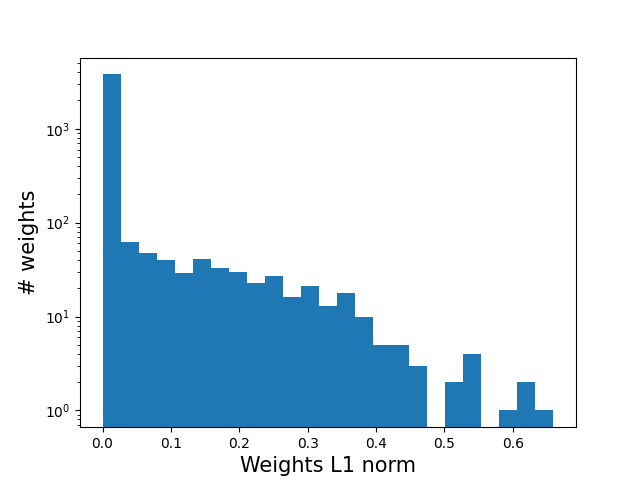}%
    \caption{}
    \label{fig:l1_decoder}
  \end{subfigure}\hfill
  
  \caption{Top panels: visual inspection of the outputs from the last convolutional layer of the encoder $\mathcal{E}_I$ (a) and the first convolutional layer of the decoder $\mathcal{D}_I$ (b), corresponding to the input image (c). The outputs reveal significant redundancy since several encoder outputs carry no meaningful data about the input image, focusing instead solely on the black background, which corresponds to a predominant number of pixels. Notice, however, how in both the encoder and the decoder the circles and the line are clearly separated, with certain filters dedicated to detecting/reconstructing the circles and others exclusively focusing on the line. This indicates that the network is disentangling the degrees of freedom related to the line from those concerning the circles. 
  Bottom panels: $L^1$ norms of the weights in the first fully connected layer of the encoder $\mathcal{E}_I$ (d) and the last fully connected layer of the decoder $\mathcal{D}_I$ (e). The distribution of $L^1$ norms in the other fully connected layers is similar. The redundancy observed in the data representation in (a) and (b) is reflected by a large number of weights possessing very small norm.}
      \label{fig:weights_analysis}
\end{figure}

The possibility to remove a significant number of weights without degrading reconstruction quality is consistent with the visual inspection of \Cref{fig:weights_analysis}. This figure demonstrates that in the encoder, most information about the original images is effectively carried by only two filters, those at positions (6,2) and (7,2) in the grid of \Cref{fig:last_encoder}. Beyond a pruning ratio of $p=0.996$, the quality of the reconstruction suddenly degrades, which is evident from both the reconstruction loss and the reconstructed images, as shown in \Cref{fig:rec_loss_pruning}. This restriction in the network's capacity for pruning ratios greater than $p=0.996$ is reflected by a change in the estimated intrinsic dimension, illustrated in \Cref{fig:dimension_pruning}.

\begin{figure}
\centering
  \begin{subfigure}[c]{.65\linewidth}
    \centering
    \includegraphics[width=\linewidth]{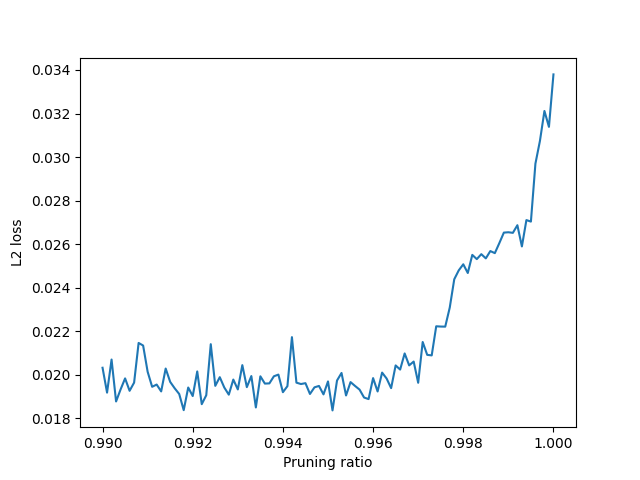}%
    \caption{}
    \label{fig:pruning_loss}
  \end{subfigure}
  \hfill
  
    \begin{subfigure}[c]{.3\linewidth}
    \centering
    \includegraphics[width=\linewidth]{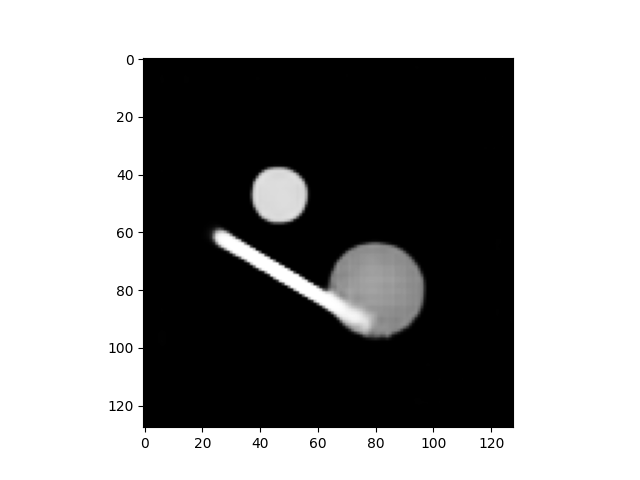}%
    \caption{}
    \label{fig:loss_pruned_pics_0}
  \end{subfigure}
      \begin{subfigure}[c]{.3\linewidth}
    \centering
    \includegraphics[width=\linewidth]{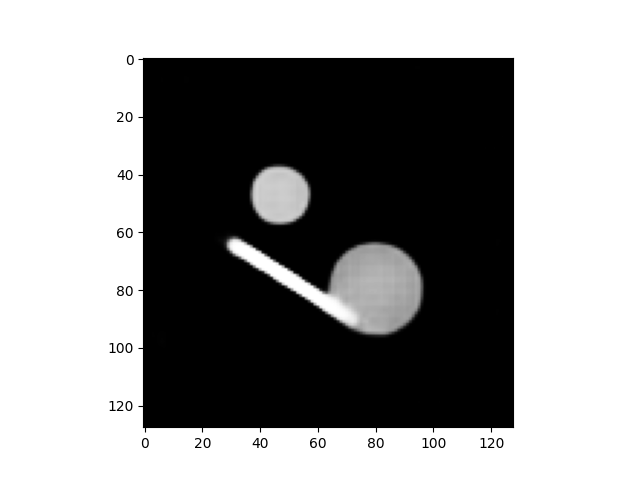}%
    \caption{}
    \label{fig:loss_pruned_pics_1}
  \end{subfigure}  
      \begin{subfigure}[c]{.3\linewidth}
    \centering
    \includegraphics[width=\linewidth]{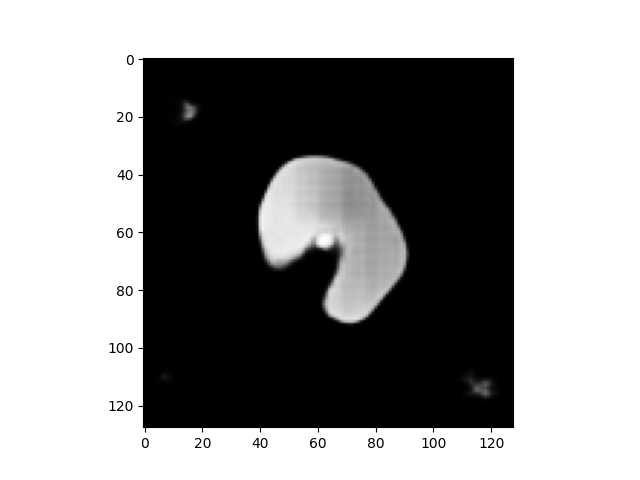}%
    \caption{}
    \label{fig:loss_pruned_pics_2}
  \end{subfigure}\hfill

\caption{Panel a): Reconstruction loss for different pruning ratio. Panels b), c) and d): Original image and its reconstructions with pruning ratios $0.996$ and $0.999$.}
\label{fig:rec_loss_pruning}
\end{figure}

\begin{figure}
\centering

    \begin{subfigure}[c]{.4\linewidth}
    \centering
    \includegraphics[width=\linewidth]{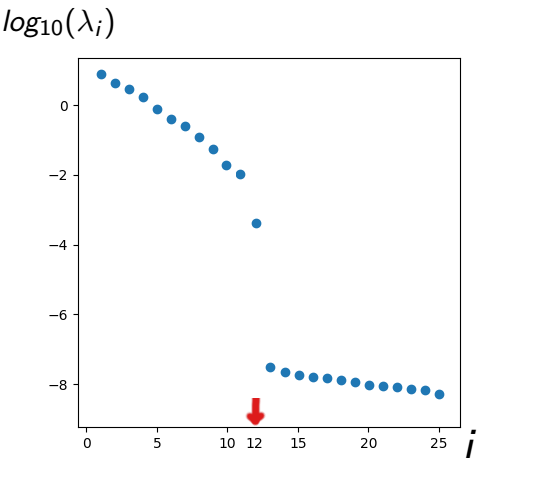}%
    \caption{}
    \label{fig:pruned_pics_original}
  \end{subfigure}
    \begin{subfigure}[c]{.4\linewidth}
    \centering
    \includegraphics[width=\linewidth]{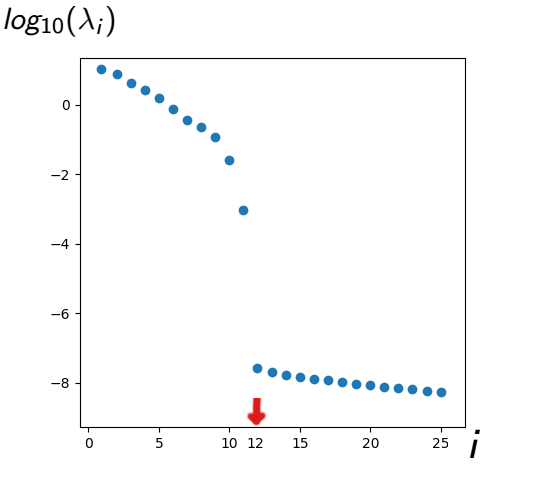}%
    \caption{}
    \label{fig:pruned_pics_0}
  \end{subfigure}
      \begin{subfigure}[c]{.4\linewidth}
    \centering
    \includegraphics[width=\linewidth]{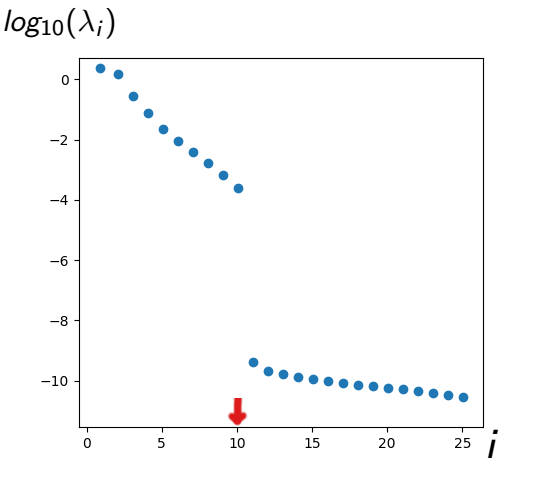}%
    \caption{}
    \label{fig:pruned_pics_1}
  \end{subfigure}  
      \begin{subfigure}[c]{.4\linewidth}
    \centering
    \includegraphics[width=\linewidth]{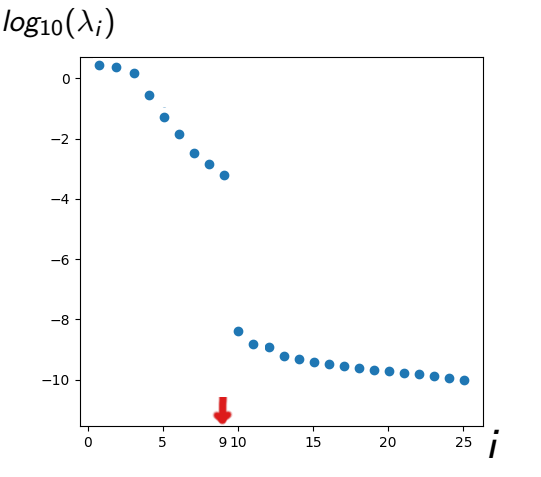}%
    \caption{}
    \label{fig:pruned_pics_2}
  \end{subfigure}\hfill

\caption{When the limiting pruning ratio is reached, the dimension of the data manifold learned by the network decreases. The panels show plots of the eigenvalues of the pullback metric for different pruning ratios. With no pruning (a) and with $p=0.9960$ (b), the estimated dimension is the correct value of $12$. However, increasing the pruning ratio leads to an estimated dimension of $10$ for $p=0.9999$ (c) and $9$ for $p=0.99995$ (d). The degradation in image reconstruction also results in a change in the shape of the eigenvalue curve.} 
\label{fig:dimension_pruning}
\end{figure}

\section{Conclusions}
Knowing the dimensionality of a data manifold is of interest in several fields of application. In the context of generative algorithms, this provides 
an estimate of the richness of generalizations since it quantifies the number of (linearly independent) ways in which an image can be altered while remaining in the space
of possible generations~\cite{Achilli24}.
This can be exploited to serve as a way
to generate in--manifold priors in 
inverse problems, which are notoriously
severely ill-conditioned. 

To perform ID estimation, in this work we have proposed to leverage
singular Riemannian metrics induced by Variational Autoencoders (VAEs). Our approach utilizes the pullback metric through the VAE decoder (or equivalently, encoder) to recover the local geometry of the underlying manifold, estimating its dimension through the numerical rank of the resulting metric tensor.
We demonstrated the effectiveness of our technique on simple low-dimensional toy datasets, such as a circle and a paraboloid, estimating correctly their known intrinsic dimensions. Subsequently, we applied this method to a more complex inverse problem in CT imaging, using a simplified version of the COULE dataset composed of circles and lines. Our analysis revealed that the image manifold has an intrinsic dimension of 12, while the sinogram manifold has an intrinsic dimension of 11, aligning well with theoretical expectations for the image data. A comparative study against existing ID estimation algorithms highlighted the superior accuracy of our method, particularly for the image dataset. We also discussed the practical considerations for constructing an atlas of local charts using mixtures of $\beta$-VAEs, noting the importance of selecting an appropriate number of charts based on the estimated ID and computational efficiency. Furthermore, our investigation into network pruning revealed that VAEs exhibit significant redundancy, allowing for substantial weight compression without immediate degradation in reconstruction quality. Crucially, we found that the intrinsic dimension estimate may serve as a robust indicator of the network's capacity and the quality of the learned manifold representation; beyond a critical pruning threshold, the estimated ID decreases, signaling a degradation in the manifold's fidelity.
Future work will further
investigate the idea that 
the pullback metric on the latent space provides information about the geometry of the manifold. 
Namely, one could explore the integration with diffusion models since one could   
treat Brownian motion directly on the manifold of reduced dimension manifold,
making these powerful methods computationally much more ``democratic''.

\vspace*{.8cm}
\noindent \textbf{Funding.} This research has been partially performed in the framework of the MIUR-PRIN Grant 20225STXSB ``Sustainable Tomographic Imaging with Learning and rEgularization'' and  GNCS Project CUP E53C24001950001 ``Metodi avanzati di ottimizzazione stocastica per la risoluzione di problemi inversi di imaging''. Paola Causin is member of the Italian group GNCS (Gruppo Nazionale Calcolo Scientifico) of INdAM.


\newcommand{\etalchar}[1]{$^{#1}$}

\clearpage

\appendix

\section{Architecture of the sub--networks}

\subsection{Embedding networks}

The architecture of the two embedding subnetworks $\mathcal{W}_S$ and $\mathcal{W}_I$ is reported in \Cref{fig:embedding_architecture}. Note that instead of using a max or average pooling layers to reduce the dimensions in the encoder, we downsample by means of convolutional layers that learn to retain the most important features. Similarly, in the decoder we upsample using a transposed convolutional layer. The dimension of the two embedding spaces must be way less than the dimension of the input since the role of $\mathcal{W}_S$ and $\mathcal{W}_I$ is indeed to perform a preliminary, possibly lossless, gross dimensional reduction. Apart from this criterion, the choice of the dimension of the latent space is arbitrary, since we do not preliminarily know the ID of the data manifold. The dimensions of the two
latent spaces for images and sinograms are set to be equal.

\subsubsection{Embedding Radon and inverse Radon networks}

The two subnetworks  $\widetilde{\mathcal{R}}$ and $\widetilde{\mathcal{R}}^{-1}$ are composed by 3 blocks of fully connected layer linked together by batch normalization layers to speed up and stabilize the training, as illustrated in \Cref{fig:radon_architecture}.

\subsubsection{Mixtures of VAEs}

The general architecture for the mixtures of VAEs is the same of \cite{ManifoldVAEs}. The VAEs of the mixture are composed by 4 coupling blocks as discussed in \Cref{subsec:inns}, each one built with 3 fully connected layers.

\begin{figure}[b!]
\begin{center}
	\begin{subfigure}[t]{0.5\textwidth}
		\centering
		\includegraphics[width=1.2\linewidth]{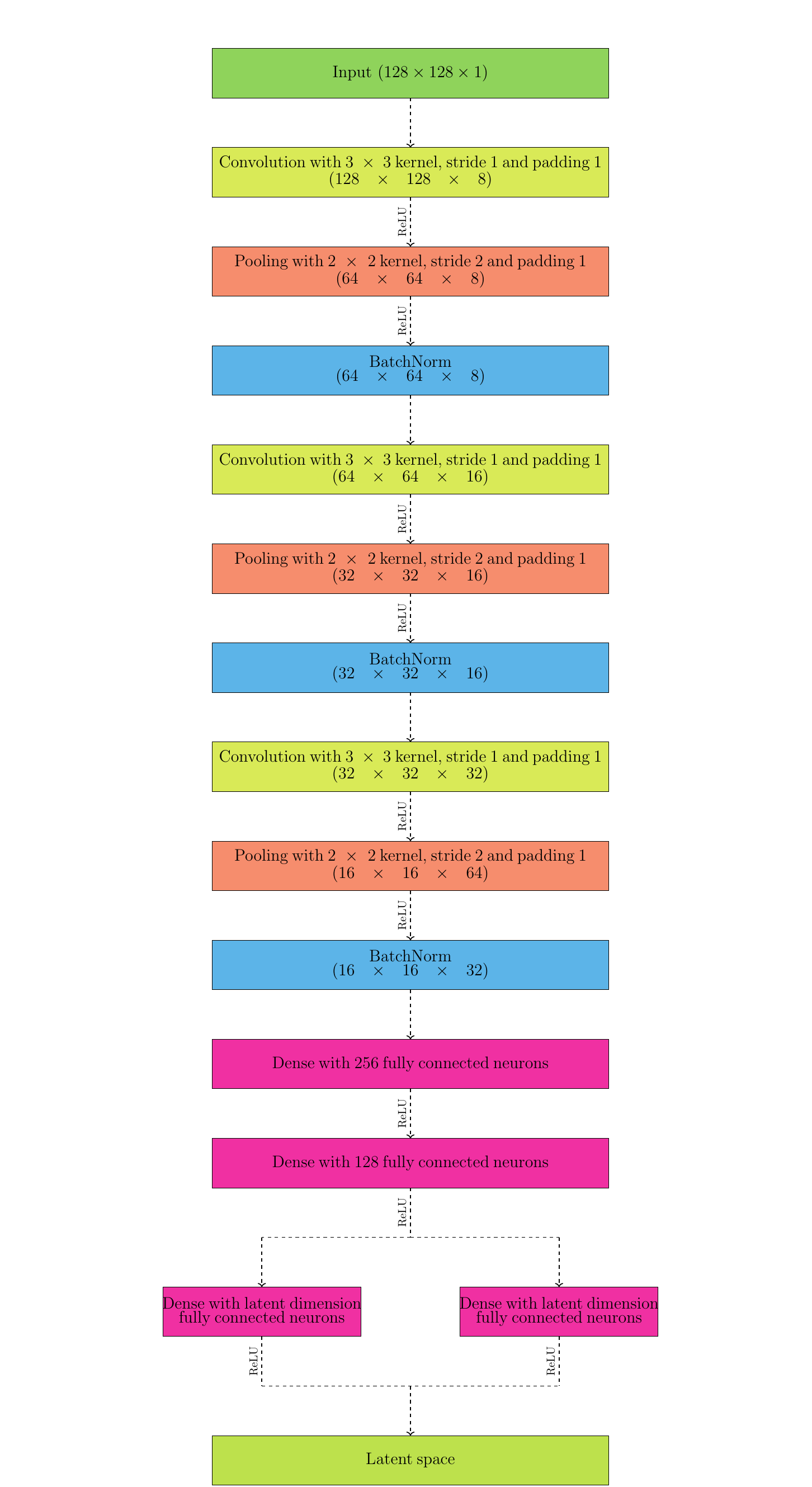}
		\caption{}
		\label{fig:null_deformations_a}
	\end{subfigure}%
	\begin{subfigure}[t]{0.35\textwidth}
		\centering
		\includegraphics[width=1.\linewidth]{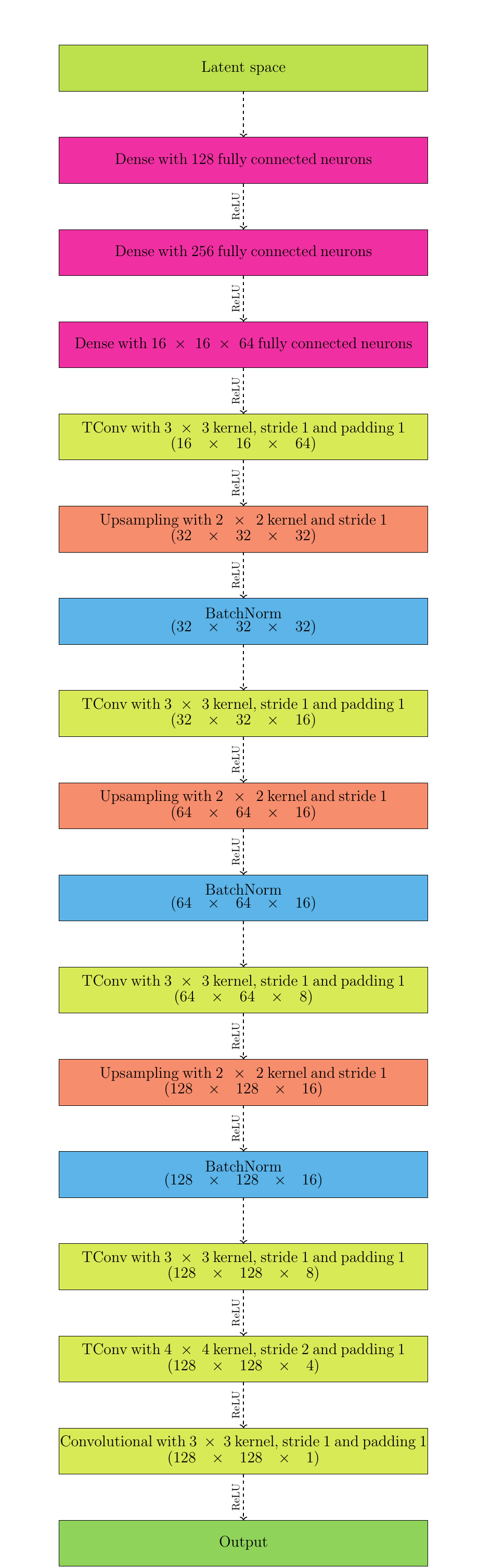}
		\caption{}
		\label{fig:null_deformations_b}
	\end{subfigure}
\end{center}
\caption{Structure of the encoder (a) and of the decoder (b) of the embedding subnetworks.}
\label{fig:embedding_architecture}
\end{figure}

\newpage

\begin{figure}[!b]
\begin{center}
\includegraphics[scale=.4]{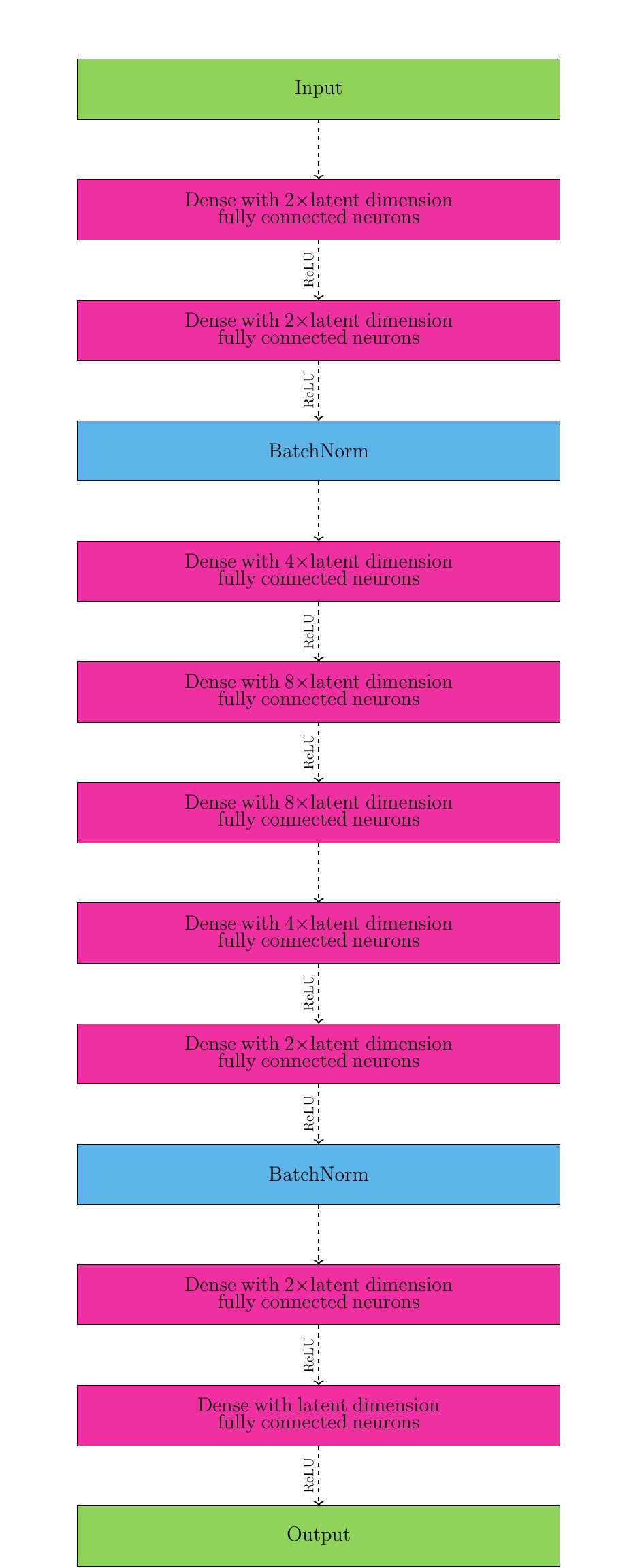}
\caption{Structure of the fully connected network employed for $\widetilde{\mathcal{R}}$ and $\widetilde{\mathcal{R}}^{-1}$.}
\label{fig:radon_architecture}
\end{center}
\end{figure}

\end{document}